\documentclass{l4dc2026}

\title[Robust Out-of-Distribution MPC via Conformalized System Level Synthesis]{Safety Beyond the Training Data: Robust Out-of-Distribution MPC via Conformalized System Level Synthesis}
\usepackage{times}
\usepackage{graphics}
\usepackage{sidecap}
\usepackage{booktabs}
\usepackage{multirow}

\usepackage{amsmath,amsfonts,bm}

\def\1{\bm{1}}

\newcommand{\calib}{\mathcal{D}_{\mathrm{calib}}}

\newcommand{\tvd}{d_\textrm{TV}}

\DeclareMathAlphabet{\mathsfit}{\encodingdefault}{\sfdefault}{m}{sl}
\SetMathAlphabet{\mathsfit}{bold}{\encodingdefault}{\sfdefault}{bx}{n}

\usepackage{multicol}
\usepackage{wrapfig}
\usepackage{algorithm,algpseudocode}
\usepackage{caption}
\usepackage{subcaption}

\newcommand{\defeq}{:=}

\newenvironment{proofsketch}{%
  \proof}{\endproof}

\renewcommand{\l}{_\textrm{lin}}

\newcommand{\methodnospace}{CP-SLS-MPC}
\newcommand{\method}{CP-SLS-MPC }

\newcommand{\Z}{\mathbf{z}}
\newcommand{\V}{\mathbf{v}}

\renewcommand{\nu}{n_\mathrm{u}}

\newcommand{\Qf}{P}

\newcommand{\f}{\mathrm{f}}
\newcommand{\covariance}{\mathbf{\Sigma}}

\newcommand{\T}{^\top}

\newcommand{\Px}{\bm{\Phi}^{\mathrm{x}}}
\newcommand{\Pu}{\bm{\Phi}^{\mathrm{u}}}
\newcommand{\F}{\mathcal{F}}

\author{%
 \Name{Anutam Srinivasan}$^{1}$ \Email{asrinivasan350@gatech.edu}\\
 \Name{Antoine Leeman}$^{2}$ \Email{aleeman@ethz.ch}\\
 \Name{Glen Chou}$^1$ \Email{chou@gatech.edu}\\
 \addr $^1$Georgia Institute of Technology. $^2$ ETH Zurich.%
}
\usepackage{xcolor}

\begin{document}

\maketitle
\vspace{-20pt}
\begin{abstract}%
We present a novel framework for robust out-of-distribution planning and control using conformal prediction (CP) and system level synthesis (SLS), addressing the challenge of ensuring safety and robustness when using learned dynamics models beyond the training data distribution. We first derive high-confidence model error bounds using weighted CP with a learned, state-control-dependent covariance model. These bounds are integrated into an SLS-based robust nonlinear model predictive control (MPC) formulation, which performs constraint tightening over the prediction horizon via volume-optimized forward reachable sets. We provide theoretical guarantees on coverage and robustness under distributional drift, and analyze the impact of data density and trajectory tube size on prediction coverage. Empirically, we demonstrate our method on nonlinear systems of increasing complexity, including a 4D car and a {12D} quadcopter, improving safety and robustness compared to fixed-bound and non-robust baselines, especially outside of the data distribution.
\end{abstract}


\vspace{-12pt}
\section{Introduction}
\vspace{-3pt}

\looseness-1Provably-safe planning under unknown dynamics is crucial for real-world robotics. This poses difficulties for classical planners, which assume precise models; in practice, model error and uncertainty are unavoidable.
Data-driven planners can learn dynamics from data but may exploit model inaccuracies, leading to unsafe behavior. 
This motivates uncertainty-constrained planners that quantify prediction error between true and learned dynamics to ensure robust, safe execution \citep{DBLP:journals/ral/KnuthCOB21, DBLP:conf/icra/KnuthCRM23, wasiela2024learning}. 
These methods estimate spatially-varying model error bounds to compute reachable tubes, which are used to constrain closed-loop behavior within a ``trusted domain" where the model is accurate. This domain is typically selected as a bounded region around the training data, preventing generated plans from straying far from the data. However, in practice, data scarcity often forces robots to operate temporarily out of distribution (OOD), tolerating degraded performance while maintaining safety.
Moreover, these planners are slow and do not exploit gradient information to efficiently optimize closed-loop behavior that actively reduces model uncertainty.

\looseness-1To address these gaps, we propose \methodnospace, a fast optimization-based robust feedback motion planner that uses the learned dynamics to jointly plan open-loop nominal trajectories and closed-loop tracking controllers via robust model predictive control (MPC). Our method uses weighted conformal prediction (CP) to bound the model error and system level synthesis (SLS) \citep{anderson2019system} to compute reachable sets, which are used to constrain the true closed-loop system to be safe with high probability, enabling robust temporary OOD operation. \method updates the model error bounds with online data to encourage cautious behavior when OOD, such as staying further from obstacles. 
\method also uses gradient data from a model error predictor to efficiently guide plans toward low-error regions, enabling near-real-time replanning. 
Our contributions are:

\vspace{-8pt}
\begin{itemize}
    \itemsep-0.35em
    \item A fast optimization-based MPC method with learned dynamics ensuring finite-sample, probabilistic safety of the true system using state-control-dependent CP error bounds in the MPC.
    \item Theoretical analysis for high-probability safety and bounding the OOD coverage gap.
    \item Numerical validation on a simulated nonlinear 4D car and 12D quadcopter, improving safety and prediction accuracy even in OOD scenarios over uncertainty-unware MPC baselines.
\end{itemize}

\vspace{-25pt}
\section{Related Work}
\vspace{-4pt}

\paragraph{Predictive Control}
\looseness-1Learning-based MPC unifies disturbance estimation with planning \citep{hewing2020learning}, but typically lacks state-dependent error bounds and cannot steer plans towards low-error regions. Robust \citep{rawlings2017model} and tube MPC \citep{mayne2005robust} ensure constraint satisfaction under disturbance, but in typical formulations, directly enforcing constraints on the closed-loop trajectory is nonconvex, leading to conservative overapproximations \citep{singh2018robust}. In contrast, SLS \citep{anderson2019system,chen2024robust}, or disturbance-feedback MPC \citep{GOULART2006523}, provides a convex parameterization of closed-loop responses for linear time-varying (LTV) systems, enabling constraint satisfaction under disturbance \citep{bartos2025stochastic}. SLS also extends to data-driven settings \citep{xue2021data,furieri2022neural} and nonlinear dynamics \citep{leeman2025robust_TAC}, where the problem is decomposed into an optimization over nominal trajectories and closed-loop LTV error dynamics. In this work, we inform SLS with state-dependent, CP-calibrated error bounds to enable safe planning toward low-error regions.

\vspace{-8pt}
\paragraph{Planning with Learned Dynamics} 
\looseness-1Many model-based reinforcement learning methods plan with learned models but lack safety guarantees \citep{DBLP:journals/ftml/MoerlandBPJ23}. Some methods estimate reachable tubes for safe exploration \citep{berkenkamp2016safe}, but assume a controller is given. The most closely related method \citep{DBLP:conf/cdc/ChouOB21, chou2022safe} safely plans with learned dynamics by constraining invariant tubes to lie in a ``trusted domain" near training data. This demands dense data coverage; in contrast, we enable the robot to move beyond the training data while assuring probabilistic safety. Subsequent model error-aware planners \citep{DBLP:conf/corl/SuhCDYGT23} improve planning speed but lack formal guarantees. Also closely related is \cite{marques2024quantifyingaleatoricepistemicdynamics}, which bounds model error via CP but cannot extrapolate OOD as it uses \textit{exchangeable} CP. Other methods aim to improve dynamics predictions in OOD scenarios \citep{DBLP:journals/corr/abs-2403-12245} but do not focus on planning.

\vspace{-6pt}
\paragraph{Conformal Prediction (CP)}
\looseness-1CP \citep{vovk2005algorithmic} is a distribution-free method for constructing prediction regions from data from black-box models, making it well-suited for safe trajectory prediction and planning \citep{DBLP:journals/ral/LindemannCSP23, DBLP:conf/l4dc/DixitLWCPB23, DBLP:conf/nips/SunJQNKS23, DBLP:conf/cdc/MuthaliSDLRFT23, muenprasitivej2025probabilistically}. The most closely related methods are \cite{DBLP:conf/l4dc/ContrerasSS25}, which triggers a safe fallback policy in OOD scenarios for MPC, and \cite{DBLP:conf/cdc/CheeHP23, DBLP:conf/l4dc/CheeSHP24}, which use WCP to address time-series correlations in MPC for tracking fixed trajectories. However, these methods do not use state-dependent uncertainty, and form prediction sets using a fixed spatially-invariant norm ball. Thus, the MPC cannot distinguish between state-space regions where the model is accurate and those where it is inaccurate, thereby degrading performance. In contrast, our method designs trajectories that are valid OOD and actively encourages uncertainty reduction, steering the MPC toward regions of lower model error, using a learned state-dependent error covariance matrix.

\vspace{-10pt}
\section{Preliminaries and Problem Statement}\label{sec:preliminaries}
\vspace{-4pt}

Let $f: \mathcal{X} \times \mathcal{U} \rightarrow \mathcal{X}$ be the true unknown deterministic nonlinear discrete-time dynamics,

\vspace{-9pt}
\begin{equation}\label{eq:dynamics}
    x_{k+1} = f(x_k, u_k) = \hat f(x_k, u_k) + \varepsilon(x_k, u_k),
\end{equation}
\vspace{-13pt}

\looseness-1\noindent where $\mathcal{X} \subseteq \mathbb{R}^{n_x}$ is the state space and $\mathcal{U} \subseteq \mathbb{R}^{n_u}$ is the control
space. Let $\hat f: \mathcal{X} \times \mathcal{U} \rightarrow \mathcal{X}$ denote an approximate learned dynamics model, where $\varepsilon: \mathcal{X} \times \mathcal{U} \rightarrow \mathcal{X}$ is the dynamics error. We model $\hat f$ with a neural network (NN), though our method is agnostic to model class. The goal is to reach a target set $\mathcal{X}_\textrm{f} \subseteq \mathcal{X}$ while satisfying pointwise-in-time state and control constraints $(x_k, u_k) \in \mathcal{F}$, where

\vspace{-18pt}
\begin{equation}\label{eq:constraints}
    \mathcal{F} := \{(x, u) \in \mathbb{R}^{n_x + n_u} \mid g_i(x,u) + b_i \le 0,\quad \forall i\in\{1,\ldots,n_c\}\text{,\quad for }g_i:\mathbb{R}^{n_x + n_u} \to \mathbb{R}\}.
\end{equation}
\vspace{-16pt}

\noindent Let $\mathcal{B}^m \subseteq \mathbb{R}^m $ be the unit 2-norm ball. Let the Minkowski sum be $A \oplus B := \{a + b\mid a \in A, b \in B\}$.

\vspace{-5pt}

\paragraph{Conformal Prediction (CP)}
We use weighted CP to convert raw non-conformity scores, quantifying a mismatch between between $\hat f(x,u)$ and $f(x,u)$, into a calibrated threshold that yields $(1-\alpha)$ coverage \citep{vovk2005algorithmic}. Let $\calib:={(x_i,u_i,f(x_i,u_i))}_{i=1}^n$ and choose a nonconformity score function  $s:\mathcal{X}\times\mathcal{U}\times\mathcal{X}\to\mathbb{R}_{\ge0}$. Set $s_i:=s(x_i,u_i,f(x_i,u_i))$. Calibration proceeds in three steps. We
(i) form the weighted empirical distribution $\hat{S}$ with weights, $\{\tilde{w}_{1}, \ldots, \tilde{w}_{n}, \tilde w_\textrm{test}\} \subseteq [0,1]^{n+1}$,

\vspace{-8pt}
\begin{equation}\label{eq:empirical_distribution}\textstyle
\hat{S}:=\sum_{i=1}^n \tilde w_i\delta_{s_i}+\tilde w_{\textrm{test}}\delta_{+\infty},\quad \textrm{s.t. }
\tilde w_{\textrm{test}}+\sum_{i=1}^n \tilde w_i=1,
\end{equation}
\vspace{-13pt}

\noindent where $\delta_{s_i}$ is the Dirac delta centered at $s_i$, (ii) calibrate the threshold $q_{1-\alpha}:=Q_{1-\alpha}(\hat{S})$ with the $1-\alpha$ quantile $Q_{1-\alpha}(\cdot)$, and (iii) define the conformal prediction set (CP set) at $(x_{\textrm{test}},u_{\textrm{test}})$ as

\vspace{-10pt}
\begin{equation}\label{eq:prediction_set}
    \mathcal{C}(x_\textrm{test}, u_\textrm{test}; q_{1-\alpha}):= \{x \in \mathcal{X} \mid s(x_\textrm{test}, u_\textrm{test}, x) \leq q_{1-\alpha}\}.
\end{equation}
\vspace{-15pt}

\noindent\looseness-1 Traditional (non-exchangeable) CP uses uniform weights, $\tilde w_i = \tilde w_\text{test} = \frac{1}{n+1}$,  and assumes exchangeability, which may fail under spatio-temporal distribution shifts; WCP does not assume this. 
We denote $(s_1, s_2, \ldots, s_\textrm{test}) \sim S^\textrm{test}$; moreover, let $S^i$ be the joint distribution obtained by replacing $s_i$ with the unknown test score, i.e., $(s_1, s_2, \ldots, s_{i-1}, s_\textrm{test}, s_{i+1}, \ldots, s_n) \sim S^i$. WCP guarantees
\vspace{-5pt}
\begin{equation}\label{eq:cp_guarantee}
\Pr\big[f(x_{\textrm{test}},u_{\textrm{test}})\in \mathcal{C}(x_{\textrm{test}},u_{\textrm{test}};q_{1-\alpha})\big]
\ge 1-\alpha-\textstyle\sum_{i=1}^n \tilde w_i\tvd(S^{i},S^\textrm{test}),
\end{equation}
\vspace{-17pt}

\looseness-1\noindent where $d_\textrm{TV}(\cdot)$ is the total-variation distance and $\textstyle\sum_{i=1}^n \tilde w_i\tvd(S^{i},S^\textrm{test})$ quantifies the coverage gap due to distribution shift (it vanishes under exchangeability). In practice, the weights can be selected to trade off coverage gap size $\sum_i \tilde w_i\tvd(S^{i},S^\textrm{test})$ and CP set size $|\mathcal{C}(x_{\textrm{test}},u_{\textrm{test}};q_{1-\alpha})|$, yielding uncertainty sets suitable for robust MPC with learned dynamics. 
WCP is essential for planning with learned dynamics to account for (1) distribution shifts between training and execution data, (2) state/control correlations during execution, and (3) visits to OOD regions of the state/control space.

\vspace{-8pt}
\paragraph{SLS} SLS \citep{anderson2019system} is a robust control framework for uncertain LTV systems

\vspace{-10pt}
\begin{equation}\label{eq:ltv_dynamics}
x_{k+1} = A_kx_{k} + B_ku_{k} + E_k\xi_k    
\end{equation}
\vspace{-15pt}

\looseness-1\noindent where $A_k\in\mathbb{R}^{n_x \times n_x}$, $B_k\in\mathbb{R}^{n_x \times n_u}$, and $E_k \in \mathbb{R}^{n_x \times n_x}$, and $\xi_k \in \mathcal{B}^{n_x}$ is a disturbance. SLS can be used to jointly optimize a nominal trajectory, a tracking controller, and a closed-loop reachable tube overapproximation computed by propagating worst-case disturbances over time. In particular, we denote a nominal length-$T$ state trajectory $\textbf{z} := \{z_1, \ldots, z_T\} \in \mathbb{R}^{n_x T}$ and control trajectory $\textbf{v} := \{v_1, \ldots, v_T\} \in \mathbb{R}^{n_u T}$ satisfying the nominal LTV dynamics in \eqref{eq:nominal_ltv}. To track the nominal $(\textbf{z}, \textbf{v})$, SLS defines a disturbance-feedback controller \eqref{eq:ctrl_sls} and closed-loop state response \eqref{eq:state_sls}, where $\Pu_{k,j} \in \mathbb{R}^{n_u \times n_x}$ and $\Px_{k,j} \in \mathbb{R}^{n_x \times n_x}$ define the closed-loop response of \eqref{eq:ltv_dynamics} as an affine function of $\{\xi_k\}$,
\vspace{-30pt}
\begin{multicols}{3}
    \begin{equation}\label{eq:nominal_ltv}
    z_{k+1} = A_k z_k + B_k v_k,
\end{equation}

    \begin{equation}\label{eq:ctrl_sls}
    u_k = v_k + \textstyle\sum_{j=1}^{k-1}\Pu_{k,j}\xi_j,
  \end{equation}

  \begin{equation}\label{eq:state_sls}
    x_k = z_k + \textstyle\sum_{j=1}^{k-1}\Px_{k,j}\xi_j.
  \end{equation}
\end{multicols}
\vspace{-5pt}

\looseness-1\noindent To ensure that \eqref{eq:ctrl_sls}-\eqref{eq:state_sls} represent the closed-loop state/control response, $\Px_{k,j}$ and $\Pu_{k,j}$ must satisfy 

\vspace{-8pt}
\begin{equation}\label{eq:sls_affine}
    \Px_{k+1,j} = A_k\Px_{k,j} + B_k\Pu_{k,j},\qquad \Px_{k+1,k} = E_k.
\end{equation}
\vspace{-13pt}

\noindent With SLS, we can guarantee that the closed-loop trajectories of \eqref{eq:state_sls} and \eqref{eq:ctrl_sls} remain in the reachable tubes 
$\mathcal{R} := \{\mathcal{R}_k:= (\mathcal{R}_k^x, \mathcal{R}_k^u)\}_{k=1}^T$, i.e., for all $k \in \{1, \ldots, T\}$, $x_k \in \mathcal{R}_k^x$ and $u_k \in \mathcal{R}_k^u$, where

\vspace{-5pt}
\begin{equation}\label{eq:sls_tubes}\textstyle
    \mathcal{R}^{x}_k := (\bigoplus_{j=0}^{k-1}\Px_{k,j}\mathcal{B}^{n_x})\oplus \{z_k\}, \quad \mathcal{R}^{u}_k := (\bigoplus_{j=0}^{k-1}\Pu_{k,j}\mathcal{B}^{n_x})\oplus \{v_k\}.
\end{equation}
\vspace{-15pt}

\noindent SLS has also been extended to control nonlinear dynamics \citep{leeman2025robust_TAC,zhan2025robustly} with robust satisfaction of nonlinear constraints, both via optimized linearization error overbounds. For nonlinear dynamics $\hat f$ of the form \eqref{eq:dynamics}, we can obtain an LTV approximation
using the Jacobian linearizations of $\hat f$ around $(\textbf{z}, \textbf{v})$: $A_k:= \bm{\nabla}_x \hat{f}\left(z_k, v_k\right),~\text{and}~B_k:=\bm{\nabla}_u \hat{f}\left(z_k, v_k\right)$.
To robustly satisfy nonlinear constraints of the form \eqref{eq:constraints}, we can enforce the tightened constraints \eqref{eq:nonlinear_constraint}, obtained by linearizing $g_i$, for all $i \in \{1,\ldots, n_c\}$, around the nominal trajectory \eqref{eq:nominal_ltv}: 
\begin{equation}\label{eq:nonlinear_constraint}\textstyle
\sum_{j=0}^{k-1}\big\Vert\big[\nabla_xg_i(z_k,v_k)\Px_{k,j}, \nabla_ug_i(z_k,v_k)\Pu_{k,j}\big]\big\Vert_{2} + g_i(z_k,v_k) + b_i +g_i^\textrm{lin}(z_k,v_k) \leq 0,
\end{equation}
where $g_i^\textrm{lin}(z_k,v_k)$ is a linearization error bound term (details are omitted for brevity; see \cite[Eq. (22)]{zhan2025robustly}). This ensures that the closed-loop dynamics satisfy \eqref{eq:constraints} for all $k \in \{1,\ldots,T\}$.

\paragraph{Problem Statement}

We assume that we are given an offline dataset $\mathcal{D} := \{(x_i, u_i, f(x_i, u_i))\}_{i=1}^{N}$, where each $(x_i, u_i)$ is drawn from a joint distribution $Y$. We partition $\mathcal{D}$ into a \textit{training dataset} $\mathcal{D}_\textrm{train} := \{(x_i, u_i, f(x_i, u_i))\}_{i=1}^{N_\textrm{train}}$, to train the learned dynamics $\hat f$ \eqref{eq:dynamics} and a model error predictor $\mathbf{\Sigma}: \mathcal{X} \times \mathcal{U} \rightarrow \mathbb{R}^{n_x \times n_x}$, described in Sec. \ref{sec:methods}, and \textit{calibration dataset} $\mathcal{D}_\textrm{calib} := \{(x_i, u_i, f(x_i, u_i)\}_{i=1}^{N_\textrm{calib}}$. Our goal is to 1) estimate calibrated error bounds on the learned dynamics $\hat f$ that hold with high probability and to 2) use these bounds within a robust MPC framework, with a prediction horizon of $T$ steps, to guarantee robust constraint satisfaction with high probability. Specifically, we solve: 

\noindent\textit{\underline{Problem 1: Learning and Calibration.}} Using $\mathcal{D}_\textrm{train}$, learn an approximate dynamics model $\hat f: \mathcal{X} \times \mathcal{U} \rightarrow \mathcal{X}$. Using $\hat f$ and $\calib$, provide a calibrated high-confidence bound $\mathcal{E}: \mathcal{X} \times \mathcal{U} \rightarrow 2^\mathcal{X}$ on the dynamics error $\varepsilon: \mathcal{X}\times \mathcal{U}\rightarrow \mathbb{R}^{n_x}$, such that $\textrm{Pr}[\bigcap_{k=1}^T \big(\varepsilon(x_k,u_k) \in \mathcal{E}(z_k,v_k)\big)] \ge 1-\tilde\alpha$, for nominal state $(z_k, v_k)$ and for all reachable state/controls $(x_k,u_k) \in \mathcal{R}_k:= (\mathcal{R}_k^x, \mathcal{R}_k^u)$, as defined in \eqref{eq:sls_tubes}. Here, $\alpha \in(0,1)$ is user-specified and $\tilde \alpha \in [\alpha, 1)$ absorbs the coverage gap in \eqref{eq:cp_guarantee}.


\looseness-1\noindent\textit{\underline{Problem 2: Planning.}} Using the learned model $\hat f$ and error bound $\mathcal{E}$ of Prob. 1, solve a robust MPC problem for $\hat f$ ensuring that the constraints \eqref{eq:constraints} are satisfied for the true closed-loop dynamics (i.e., under $f$ in \eqref{eq:dynamics}) with probability at least $1-\tilde \alpha$, jointly over $T$ prediction steps, for a single MPC solve, and with probability at least $1-\hat \alpha$ for $R$ MPC solves (i.e., $R-1$ replanning steps), where $\hat \alpha \in (\tilde \alpha, 1)$. 

\vspace{-10pt}
\section{Method}\label{sec:methods}
We outline our method, \methodnospace, which uses WCP for high-probability state-control-dependent model error bounding (Sec. \ref{sec:cp_err_bound}) and integrates it into an uncertainty-constrained SLS-based robust MPC optimization (Sec. \ref{sec:sls_plus_cp_bound}). Finally, we discuss practical implementation (Sec. \ref{sec:scp}). 

\vspace{-12pt}
\subsection{Model Training and Model Error Bounding via Conformal Prediction}\label{sec:cp_err_bound}
\looseness-1To tackle Prob. 1, we first train the dynamics $\hat f$ \eqref{eq:dynamics} on the training dataset $\mathcal{D}_\textrm{train}$ by minimizing a mean-squared-error loss $\textstyle\sum_{k=1}^{N_\textrm{train}} \Vert f(x_k, u_k) - \hat f(x_k, u_k)\Vert_2$. Then, for any nominal state $z \in \mathcal{X}$ and control input $v \in \mathcal{U}$, we seek to bound the error of the learned dynamics model $\Vert \hat{f}(z, v)-f(z,v)\Vert_2 $. To encode spatially-varying error, we train a heuristic uncertainty model $\covariance: \mathcal{X}\times \mathcal{U} \to \mathbb{R}^{n_x \times n_x}$, an NN, with a multivariate Gaussian negative-log-likelihood (MGNLL) loss \citep{dasgupta2007line}:
\begin{equation}\textstyle\small
\mathcal{L}(\covariance; f, \hat{f}, x, u) = \frac{1}{2} \left( (f(x,u) -\hat{f}(x,u))\T \covariance(x,u)^{-1} (f(x,u) -\hat{f}(x,u)) + \ln(\det(\covariance(x,u)))\right).
\end{equation}
\looseness-1MGNLL is selected to train $\mathbf{\Sigma}$ to output an uncalibrated covariance matrix which locally approximates the \textit{shape} and \textit{magnitude} of $\hat{f}$'s error dispersion at a given state-control pair. WCP's calibration procedure reconciles the inexact gaussian assumption (of MGNLL) made of the error dispersion by scaling the covariance matrix.   
To capture the model error at timestep $k$ for a nominal ($z_k, v_k)$, we adapt the non-conformity score of \cite{messoudi2022ellipsoidal}:
\begin{equation}\label{eq:non_conform}\small
s_{i,k} := s_k(x_i, u_i, f(x_i, u_i)) = \sqrt{\big(f(x_i, u_i) - \hat{f}(x_i, u_i)\big )\T\covariance(z_k,v_k)^{-1}\big(f(x_i, u_i) - \hat{f}(x_i, u_i)\big)}.
\end{equation}
Using \eqref{eq:non_conform}, we form the $(1-\alpha_k)$-confidence CP set for some nominal $(z_k, v_k)$, $\mathcal{C}(z_k, v_k; {q_{1-\alpha_k}})$, as
\begin{equation}\label{eq:ellip_prediction_set}\small
    \mathcal{C}(z_k,v_k; {q_{1-\alpha_k}}):= \{x \in \mathcal{X} \mid \sqrt{\big(x - \hat{f}(z_k,v_k)\big )\T\covariance(z_k,v_k)^{-1}\big(x - \hat{f}(z_k,v_k)\big)} \leq q_{1-\alpha_k}(z_k, v_k)\},
\end{equation}
defining $q_{1-\alpha_k}(z_k,v_k):=Q_{1-\alpha_k}(\hat{S}_k)$, where $\hat{S}_k := \sum_{i=1}^n \tilde w^k_i\delta_{s_{i,k}}+\tilde w^k_{\textrm{test}}\delta_{+\infty}$ is formed using a set of normalized weights $\tilde W^k := \{\tilde{w}_{1}^{k}, \dots, \tilde{w}_{N_\textrm{calib}}^{k}, \tilde w_\textrm{test}^k\} \subseteq [0,1]^{N_\textrm{calib}+1}$ tailored to the nominal state/control pair $(z_k,v_k)$. 
Intuitively, $\mathcal{C}(z_k,v_k; {q_{1-\alpha_k}})$ is an ellipsoid centered at the predicted next state $\hat{f}(z_k, v_k)$, with a radius given by the quantile of the weighted non-conformity scores from \eqref{eq:non_conform}. 

To construct the normalized weights $\tilde W^k$ tailored for a nominal state/control pair $(z_k,v_k)$, we first define a set of unnormalized weights $W^k = \{{w}_{1}^{k}, \dots, {w}_{N_\textrm{calib}}^{k}, w_\textrm{test}^k=1\} \subseteq [0,1]^{N_\textrm{calib}+1} $, where
\begin{equation}\label{eq:rho_weights}
    w^k_i = \rho^{||(z_k,v_k) - (x_i, u_i)||_2},
\end{equation}
\looseness-1where $\rho \in (0, 1]$, for all $(x_i, u_i)\in \calib$. The form of \eqref{eq:rho_weights} encodes the assumption that model error is state/control-dependent and changes smoothly, such that nearby pairs have similar errors. We normalize the weights $w^k$ as $\tilde w^k_i = \frac{w^k_i}{1 + \sum_{j=1}^{n}w^k_j}$ to be compatible with the WCP framework (Sec. \ref{sec:preliminaries}). 

\looseness-1Denote $(s_{1,k}, s_{2,k}, \ldots s_{i-1,k}, s_{k,k}, s_{i+1,k}, \ldots, s_{N_\textrm{calib},k}, s_{i,k}) \sim S^{i,k}$ and $(s_{1,k}, \ldots, s_{N_\textrm{calib},k}, s_{k,k}) \sim S^{k,k}$ (i.e., $S^{i,k}$ exchanges $s_{k,k}$ with $s_{i,k}$), where $s_{k,k} = s_k(z_k,v_k,f(z_k,v_k))$ is generally unknown as $f(z_k,v_k)$ may not be in $\calib$ for some candidate $(z_k, v_k)$. Via the WCP framework, \eqref{eq:cp_guarantee} ensures

\vspace{-8pt}
\begin{equation}\label{eq:cp_guarantee_f}\textstyle
    \textrm{Pr}[{f}(z_k, v_k) \in \mathcal{C}(z_{k}, v_{k};{q_{1-\alpha_k}})] \ge 1 - \alpha_k -\sum_{i=1}^{n}\tilde{w}_i^k\,\tvd(S^{i,k}, S^{k,k}).
\end{equation}
\vspace{-12pt}

\noindent Starting from \eqref{eq:cp_guarantee_f}, we see that the dynamics error is in an origin centered ellipsoid, $\varepsilon(z_k,v_k) = \big(f(z_k,v_k) - \hat f(z_k, v_k) \big) \in \mathcal{C}(z_k, v_k; q_{1-\alpha_k}) \oplus \{-\hat f(z_k, v_k)\}$, implying that it satisfies $\Pr[\varepsilon(z_k,v_k) \in \mathcal{E}(z_k,v_k)] \geq 1 - \alpha_k -\sum_{i=1}^{n}\tilde{w}_i^k\,\tvd(S^{i,k}, S^{k,k}) $, where we set $\mathcal{E}(z_k, v_k) := \mathcal{C}(z_k, v_k; q_{1-\alpha_k}) \oplus \{-\hat f(z_k, v_k)\}$. In essence, by conformalizing the heuristic uncertainty model $\mathbf{\Sigma}$, we obtain a calibrated spatially-dependent model error bound $\mathcal{E}(z_k, v_k)$ that holds with high probability. 
Equivalently, using the Cholesky factor $L(z_k,v_k)$ of $\mathbf{\Sigma}(z_k,v_k)$, i.e., $\mathbf{\Sigma}(z_k,v_k) = L(z_k,v_k)L(z_k,v_k)^\top$, \eqref{eq:ellip_prediction_set}-\eqref{eq:cp_guarantee_f}, and selecting a miscoverage rate $\alpha_k \in (0,1)$ tailored to timestep $k$, we can state that
\begin{equation}\label{eq:ball_form}\textstyle
\textrm{Pr}[\varepsilon(z_k,v_k) \in \underbrace{V(z_k,v_k)\mathcal{B}^{n_x}}_{:=\mathcal{E}(z_k,v_k)}] \ge 1 - \alpha_k -\sum_{i=1}^{N_\textrm{calib}}\tilde{w}_i^k\,\tvd(S^{i,k}, S^{k,k}),
\end{equation}\vspace{-22pt}

\begin{equation}\label{eq:ball_form_vzk}\textstyle
\text{where}\quad V(z_k,v_k) := {q_{1-\alpha_k}(z_k,v_k)L(z_k,v_k)}.
\end{equation}
\looseness-1\noindent While \eqref{eq:ball_form}-\eqref{eq:ball_form_vzk} bounds model error at $(z_k, v_k)$, the system may be perturbed during execution within the reachable tube $\mathcal{R}_k$. Thus, at every timestep $k$, we must quantify the coverage of $\mathcal{E}(z_k, v_k)$ from any $(x_k,u_k) \in \mathcal{R}_k$. 
This can be achieved by inflating the coverage gap with an additional term, yielding $\Pr[\varepsilon(x_k,u_k) \in \mathcal{E}(z_k,v_k)] \ge 1-\alpha_k- 2\sum_{i=1}^{N_\textrm{calib}} \tilde{w}_i^k \tvd(S_{i,k}, S_{k,k}) - \gamma(\mathcal{R}_k)$, where $\gamma(\mathcal{R}_k)$ is defined in Thm. \ref{thm:tube_gap}. Here, $S_{i,k}$ and ${S}_{k,k}$ are distributions such that $s_{i,k} \sim S_{i,k}$, $s_{k,k} \sim {S}_{k,k}$. 
Specifically, Thm. \ref{thm:total_coverage} derives the probability that $\varepsilon(x_k,u_k)\in\mathcal{E}(z_k,v_k)$ for all $(x_k,u_k) \in \mathcal{R}_k$, for all $k \in \{1,\ldots,T\}$, providing our solution for Prob. 1 (proofs in App. \ref{app:proofs}):

\vspace{-4pt}
\begin{theorem}\label{thm:total_coverage}
    \looseness-1Given a length-$T$ trajectory of the true dynamics $f$ \eqref{eq:dynamics} $\tau_f:= \{(x_k, u_k)\}_{k=1}^{T}$ such that $(x_k, u_k) \in \mathcal{X} \times \mathcal{U}$ for all $k \in \{1, \ldots, T\}$, desired miscoverage $\{\alpha_k\}_{k=1}^T$, total miscoverage $\sigma_k = \alpha_k + 2\sum_{i=1}^{N_\textrm{calib}} \tilde{w}_i^k \tvd(S_{i,k}, S_{k,k})+ \gamma({\mathcal{R}_k})$, and independence between  $\calib$ and $(x_k,u_k)$, we have:

    \vspace{-10pt}
    \begin{equation}\textstyle
        \Pr[\bigcap_{k=1}^{T}~\big(\varepsilon(x_k,u_k) \in \mathcal{E}(z_k,v_k)\big)] \ge 1- \sum_{k=1}^{T}\sigma_k := 1-\tilde\alpha.
    \end{equation}
\vspace{-20pt}
    
\end{theorem}
Overall, at every timestep, we predict an ellipsoid $\mathcal{E}(z_k,v_k)$ containing the one-step error $\varepsilon(x_k,u_k)$ with high probability. Next, we will use this data-driven, spatially-dependent bound $\mathcal{E}(z_k, v_k)$ to inform SLS-based robust MPC, shrinking tubes where data is dense and expanding them if OOD.

\vspace{-10pt}
\subsection{Informing SLS-based Planning with CP Model Error Bounds}\label{sec:sls_plus_cp_bound}

\looseness-1Using our conformalized model error bounds, we construct an uncertain nonlinear system $z_{k+1} = \hat{f}(z_k,v_k) + V(z_k,v_k)\xi_k$,
where $\xi_k \in \mathcal{B}^{n_x}$ and $ V(z_k,v_k)\xi_k$ represent the unknown, bounded model error at timestep $k$. 
Let $\Px$ and $\Pu$ collect all $\Px_{k,j}$ and $\Pu_{k,j}$. Then we adapt \cite{leeman2025robust_TAC} to write a nonlinear program (NLP) in \eqref{eq:nonlinear_sls} to optimize for $\Z = \{z_1,\dots,z_{T}\}$, $\V = \{v_1,\dots,v_{T-1}\}$, $\Px$, and $\Pu$, given an initial state $\bar{x}_0 \in \mathcal{X}$: 
\begin{subequations}\small
\label{eq:nonlinear_sls}
\begin{align}
\min_{\substack{\mathbf{\Phi}_{\mathbf{x}},\mathbf{\Phi}_{\mathrm{u}},\mathbf{z}, \mathbf{v}}}\quad &  J(\mathbf{z},\mathbf{v}) +  \tilde H_0(\Px, \Pu) + J_{\mathcal{X}_\textrm{f}}(\textbf{z}, \textbf{v}),\label{eq:nonlinear_objective}\\[-6pt]
\text { s.t. }\quad\quad &  {z}_{k+1}= \hat{f}(z_k,v_k),~ {z}_1 = \bar x_0,\quad\quad \forall k = 1,\ldots, T-1,\label{eq:nonlinear_nominal}\\
&\Px_{k+1,j} =A_k  \Px_{k,j} + B_k \Pu_{k,j},\ \quad \forall j = 1,\ldots, T-1,\quad \forall k = j+1,\ldots,T-1,\label{eq:SLP_nonlinear}\\
&\Px_{j+1,j} = V(z_j, v_j),\quad \forall j = 1,\ldots, T-1,\label{eq:SLP_ic_nonlinear}\\
& \textstyle\sum_{j=1}^{k}\big\Vert[\nabla_xg_{i}(z_k,v_k)\Px_{k,j}, \nabla_ug_{i}(z_k,v_k)\Pu_{k,j}]\big\Vert_{2} + g_{i}(z_k,v_k) + b_{i} + g_i^\textrm{lin}(z_k,v_k)  \leq 0,\label{eq:SLC_nonlinear}\\
&  \quad\quad \forall i =1, \ldots, n_c,\quad \forall k= 1,\ldots,T. \nonumber
\end{align}
\end{subequations}
\noindent Here, \eqref{eq:nonlinear_objective} penalizes a cost $J(\mathbf{z},\mathbf{v})$ on the nominal trajectory,  $J_{\mathcal{X}_\textrm{f}}(\textbf{z}, \textbf{v})$ penalizes reaching the terminal set $\mathcal{X}_\textrm{f}$, and $\tilde H_0(\Px, \Pu)$ penalizes the size of the reachable tubes (see \eqref{eq:regulizer}, \eqref{eq:lqr}, \eqref{eq:lqr_terminal}). 
Since $\hat f$ is more accurate in the training domain, we modify $J(\mathbf{z},\mathbf{v})$ to include an active error reduction cost $J_\mathrm{active}(\mathbf{z},\mathbf{v})$ (see App.~\ref{app:uncert_reduc}) minimizing the distance of $(\mathbf{z},\mathbf{v})$ to in-distribution calibration points. Constraint \eqref{eq:nonlinear_nominal} ensures nominal dynamic feasibility, while \eqref{eq:SLP_nonlinear} and \eqref{eq:SLP_ic_nonlinear} enforce SLS conditions \eqref{eq:sls_affine} for the LTV approximation obtained by linearizing $\hat f$ around $(\mathbf{z},\mathbf{v})$, with model error bound $V(z_j,v_j)$ from \eqref{eq:ball_form_vzk}. Constraint \eqref{eq:SLC_nonlinear} enforces robust state-control satisfaction using bounds on the closed-loop response derived by informing SLS with CP. Since SLS guarantees an overapproximation of the true reachable tube, (1) the closed-loop trajectory executing $(\mathbf{z}, \mathbf{v}, \Px, \Pu)$ over the $T$-step horizon satisfies \eqref{eq:constraints} with high probability, and (2) safety is maintained with high, though diminishing, probability over iterative replanning steps. The following result addresses Prob. 2:

\vspace{-6pt}
\begin{theorem}\label{cor:rollout_prob}
    Define the forecasted nominal trajectory {\small $\tau_{\hat f}^{(q)} := \{(z_k^{(q)}, v_k^{(q)})\}_{k=1}^{T}$} as the $T$-timestep future nominal trajectory predicted by the learned dynamics {\small$\hat f$} at the $q$th MPC solve (i.e., the $(q-1)$st replanning step) in \eqref{eq:nonlinear_nominal}. \textbf{(a)} Then the true closed-loop dynamics \eqref{eq:dynamics} under the controller defined by \eqref{eq:nonlinear_sls} and \eqref{eq:ctrl_sls} satisfies {\small $\Pr[\bigcap_{k=1}^T \big((x_k^{(q)}, u_k^{(q)})\in \mathcal{F}\big)] \geq 1 - \sum_{k=1}^{T}\sigma_k^{(q)}$}, where {\small $\{(x_k^{(q)}, u_k^{(q)})\}_{k=1}^T$} is the realized state-control trajectory if executed for the full $T$ timesteps, and \textbf{(b)} defining $\sigma_1^{(q)}$ as the miscoverage rate for the first timestep at MPC solve iteration $q$, we have that by executing \eqref{eq:nonlinear_sls} over $R$ MPC solves (i.e., $R-1$ re-planning steps), {\small $\Pr[\bigcap_{q=1}^R\big((x_1^{(q)}, u_1^{(q)})\in \mathcal{F}\big)] \geq 1 - \sum_{q=1}^{R}\sigma^{(q)}_1 := 1-\hat \alpha$}. 
\end{theorem}
\vspace{-4pt}
\noindent Overall, by solving \eqref{eq:nonlinear_sls}, we reduce model error by 1) minimizing reachable tube volume and 2) penalizing distance to data. Coupling planning with CP-informed reachability ensures high-probability safety for the true dynamics. Finally, refining the CP error set $\mathcal{C}$ with online data (Alg.~\ref{alg:robust_mpc}) reduces the coverage gap, improving empirical model error coverage.
\vspace{-8pt}

\subsection{Practical Implementation via SCP}\label{sec:scp}
\vspace{-4pt}

\looseness-1To efficiently solve the NLP \eqref{eq:nonlinear_sls}, we propose \method (Alg. \ref{alg:robust_mpc}). \method uses sequential convex programming (SCP) \citep{malyuta2022convex,messerer2021survey}. Each SCP iteration solves a second-order cone program (SOCP) (Alg. \ref{alg:robust_mpc}, line \ref{alg:lin:single}), obtained by linearizing $\hat f$ around the current nominal trajectory $(\textbf{z}, \textbf{v})$ and convexifying the constraint tightenings in \eqref{eq:SLC_nonlinear}, yielding an update $(\Z\l^*, \V\l^*)$ to the nominal trajectory. 
To solve this SOCP efficiently, we use the solver in \cite{leeman2024fast}, which iterates between solving a QP for $(\mathbf{z}, \mathbf{v})$ and Riccati recursions for $(\Px, \Pu)$ (line \ref{alg:lin:fastsls}). For efficiency, we apply a real-time iteration scheme \citep{Gros02012020}, limiting SCP to one iteration per step \citep{leeman2025guaranteed}, and omit linearization error \citep{zhan2025robustly}, which is small in practice (Sec.~\ref{sec:results}). At each step, we execute the first control input $v_1$ (line \ref{alg:lin:execute}), shift the previous solution for initialization (line \ref{alg:lin:shift}), and re-solve via SCP. Moreover, we improve OOD robustness by updating the error set $\mathcal{V}$ with online observations (line \ref{alg:lin:data}). Augmenting the error set updates the quantile $q_{1-\alpha_k}(z_k, v_k)$, which in turn alters constraint \eqref{eq:SLP_ic_nonlinear} via \eqref{eq:ball_form_vzk}, reducing the coverage gap and increasing the likelihood that the closed-loop dynamics remain within the computed reachable tubes.

\vspace{-4pt}
\begin{algorithm}[H]\small
\caption{\method}\label{alg:robust_mpc}
\begin{algorithmic}[1]
\Procedure{CP-SLS-MPC}{$\calib, \text{horizon }T, \mathcal{F}, \bar x_0, \text{ terminal set }\mathcal{X}_\f$}
    \State $t \gets 1$ and $(\Z, \V) \gets $ Nominal Solution \Comment{Initialize the problem with start state $\bar{x}_0$}
    \State $\mathcal{V} \gets \{(x_i, u_i, f(x_i,u_i) - \hat f(x_i,u_i))\}_{i=1}^{N_\textrm{calib}},\ \forall (x_i, u_i, f(x_i, u_i))\in\calib$ \Comment{Initialize error set $\mathcal{V}$}
    \State \textbf{while} {$x_t \notin \mathcal{X}_\f$} \textbf{do} \Comment{$\mathcal{X}_\f$ is included as a terminal cost \eqref{eq:lqr}, App \ref{app:lqr_app}}
        \State $\quad \bar x_0 \gets x_t$; initialize $(\Z, \V)$ with shifted previous solution \Comment{Initialize with the current state}
        \label{alg:lin:shift}\State $\quad$Linearize \eqref{eq:nonlinear_sls} around $(\Z, \V)$; compute $V(z_k, v_k)$ via \eqref{eq:ball_form_vzk}, $\mathcal{V}$; form SOCP
        \label{alg:lin:single}\State $\quad$$(\Z\l^*, \V\l^*)$, $\Px, \Pu \gets$ SOCP solution \Comment{Single iteration of \cite{leeman2024fast}}.
        \label{alg:lin:fastsls}\State $\quad$Update $(\Z, \V) \leftarrow (\Z, \V) + (\Z\l^*, \V\l^*)$.

        \State $\quad$$u_t \gets v_1$; $x_{t+1} \gets f(x_t,u_t) $ \Comment{Closed loop control; state update}
        \label{alg:lin:execute}\State $\quad \mathcal{V} \gets \mathcal{V} \cup \{(x_t,u_t, x_{t+1} - \hat{f}(x_t,u_t))\}$; $t \gets t + 1$ \Comment{Online data update}\label{alg:lin:data}
\EndProcedure
\end{algorithmic}
\end{algorithm}
\vspace{-26pt}

\section{Theoretical Analysis: Bounding the Coverage Gap}\label{sec:theory}

\looseness-1To bound the coverage gap, we assume non-conformity scores from the calibration set $\calib$ are independent of the score at the nominal point $(z_k,v_k)$, though they may come from different distributions. We also assume a Lipschitz-type bound on the distribution drift between ${S_{i,k}}$ and $S_{k,k}$, i.e.,

\vspace{-10pt}
\begin{equation}\label{eq:lipschitz}
\tvd( S_{i,k},S_{k,k}) \leq  \epsilon ||[z_k\T,v_k\T] - [x_i\T,u_i\T]||_{2},    
\end{equation}
\vspace{-13pt}

\noindent for $\epsilon \in \mathbb{R}, \epsilon > 0$, where $s_i \sim S_i$ \eqref{eq:non_conform} and $\tilde{S}$ is the non-conformity score distribution of $(z,v)$. 
\eqref{eq:lipschitz} is a practical assumption, stating that the dynamics error distribution shifts gradually over the state-control space. 
The next theorem shows a overbound on the calibration error $\tvd(S_{i,k},S_{k,k})$.
\vspace{-8pt}
\begin{theorem}\label{thm:cov_gap_thm} The \textit{Coverage Gap}$\,=\sum_{i=1}^{N_\textrm{calib}}\tilde{w}_i\,\tvd(S^{i,k}, S^{k,k})$ of our conformal predictor satisfies

\vspace{-10pt}
\begin{equation}\textstyle
    \text{Coverage Gap}\le \underbrace{\textstyle\sum_{i=1}^{N_\textrm{calib}} \left( \frac{\rho^{d_i}}{1 + \sum_{j=1}^{N_\textrm{calib}} \rho^{d_j}} \right) \cdot 2\epsilon \cdot d_i}_{\text{Tight Bound}} \le \underbrace{\textstyle 2 \epsilon \left[  \frac{d_1 }{1 - \rho^{d_\text{min}}} + \frac{d_\text{max}\rho^{d_\text{min}}}{(1-\rho^{d_\text{min}})^2}\right]}_{\text{Interpretable Bound}}, 
\label{eq:thm3_bound}\end{equation}
where we sort the distance to points in the calibration set such that $d_i = ||[z_k\T,v_k\T] - [x_i\T,u_i\T]||_{2}$, $d_1 \le d_2 \le \dots \le d_n$, $d_\text{min} = \min_{1 \leq i \leq n-1}\{d_{i+1} - d_i\}$, and $d_\text{max} = \max_{1 \leq i \leq n-1} \{ d_{i+1} - d_i\}$.
\end{theorem}
\vspace{-2pt}

\noindent Thm. \ref{thm:cov_gap_thm} (interpretable bound) shows that the gap depends on state-control calibration point density. Greater point density decreases $d_\text{min}$, increasing the gap; reducing $\rho$ mitigates this, but doing so can cause the $(1-\alpha)$-quantile to diverge. The term $\frac{d_1}{1 - \rho^{d_\text{min}}}$ shows that being closer to the training data reduces the gap, as $d_1$ is the minimal distance between $\calib$ and $(z_k,v_k)$. Further discussion of the intuition and validation of the empirical bounds is in App. \ref{app:theoretical_expansion}. Online augmentation of $\mathcal{V}$ may violate independence (due to time-correlation), but for large $N_\textrm{calib}$, Thm. \ref{thm:cov_gap_thm} suffices. As SLS requires model error to hold for all $(x,u) \in \mathcal{R}_k$, we also provide a coverage bound valid over the reachable tube:
\begin{theorem}\label{thm:tube_gap}
Under the same assumption of Theorem \ref{thm:cov_gap_thm}, we have that for all $(x,u) \in \mathcal{R}_k$, the tube around the nominal point $(z_k,v_k)$, 
    $\Pr[\varepsilon(x,u) \in \mathcal{E}(z_k,v_k)] \geq 1-\alpha_k- 2\sum_{i=1}^{n} \tilde{w}_i d_{TV}(S_{i,k}, S_{k,k})- \gamma({\mathcal{R}_k})$,
for $ \gamma({\mathcal{R}_k}) := 2\hat{\epsilon}M(\mathcal{R}_k) $, where $M(\mathcal{R}_k)$ is the maximum axis length of tube $\mathcal{R}_k$, and $\hat \epsilon \le \epsilon$ is chosen such that \eqref{eq:lipschitz} holds for all $(x, u) \in \mathcal{R}_k$.
\end{theorem}

Using Thm. \ref{thm:tube_gap}, we find that the coverage guarantees remain close to the desired coverage level $\alpha$, assuming a gradual spatial-drift of the error distribution. Furthermore, the optimization problem \eqref{eq:nonlinear_sls} seeks to minimize the tube volume in the cost function, thereby minimizing $\gamma({\mathcal{R}}_k)$. 

\vspace{-12pt}
\section{Results}\label{sec:results}
\vspace{-3pt}

We evaluate our method on robust planning with learned 4D Dubins' (see App.~\ref{app:car_dyn}) and 12D quadcopter dynamics (see App. \ref{app:quad_dyn}). Comparisons (see App. \ref{app:probs} for details) include na\"ive, uncalibrated MPC on the learned dynamics $\hat f$ without robust constraint tightenings, denoted as vanilla MPC (V-MPC), and a fixed hyper-ball variant of our method (CP-Ball), where $\covariance(z_k,v_k)$ in \eqref{eq:non_conform} is replaced with $I_{n_x}$, highlighting the benefit of state-dependent CP-Ellipsoid (our method) uncertainty sets. CP-Ball is an adaptation of \citet{DBLP:conf/l4dc/CheeSHP24}, using the L2-norm of errors as the non-conformity score. 
Direct use of \citet{DBLP:conf/l4dc/CheeSHP24} was not feasible, as it only collects online data points, requiring horizon lengths longer than 150 points for successful deployment. 
Training details, model architectures, and cost functions are in App.~\ref{app:training}. Our implementation uses Acados \citep{verschueren2020acadosmodularopensourceframework} and L4CasADi \citep{salzmann2024l4casadi} and is run on an M4 Pro MacBook (12 cores, 24 GB RAM). We report average MPC step time, prediction error $|f(x,u) - \hat f(x,u)|$, and minimum obstacle distance. All trials use horizon $T=15$ with $\alpha_k = 0.1/15$.

\begin{figure}[htbp]
    \centering
    \vspace{-8pt}
    \includegraphics[width=\textwidth]{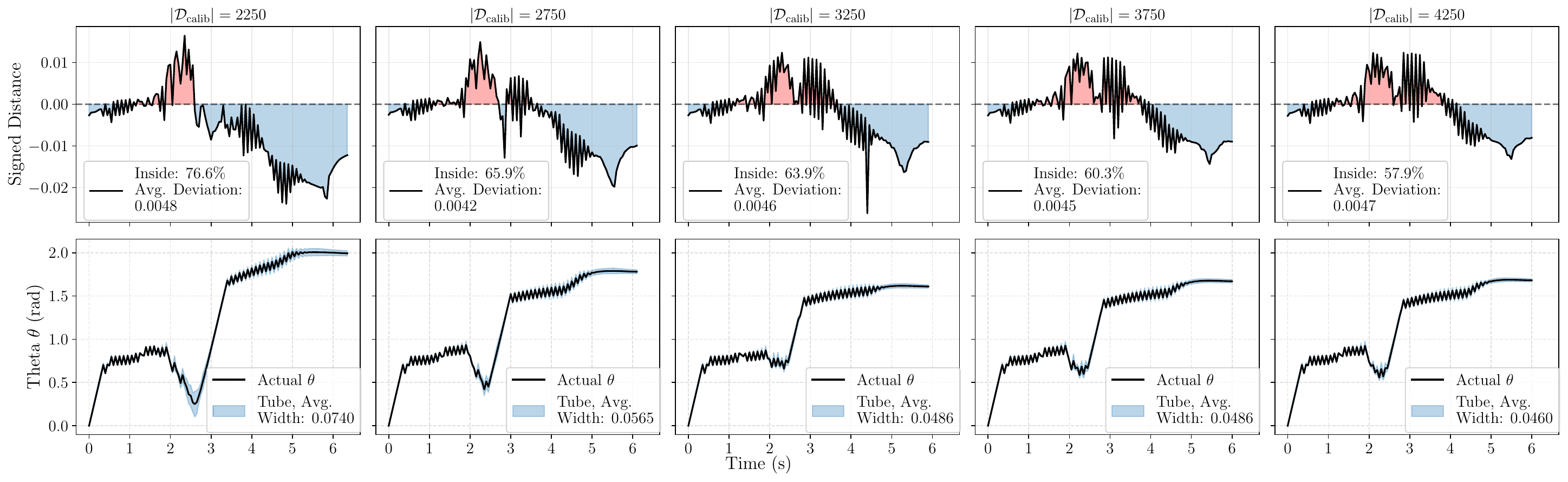}
    \vspace{-25pt}
    \caption{\textbf{OOD Car.} (Top) Minimum distance between prediction error and the ellipsoid edge is plotted for varying $\calib$ with a fixed start and goal; negative values (blue) indicate points inside the ellipsoid, i.e., valid coverage. (Bottom) Corresponding one-step $\theta$ tubes are shown. Smaller $\calib$ keeps prediction errors inside the ellipsoid more often and results in larger tubes.
    }\vspace{-15pt}
    \label{fig:OOD_tubes}
\end{figure}
\paragraph{Dubins': In Domain (ID).} First, we validate that our approach maintains performance compared to baselines in ID scenarios without a major loss of computational efficiency. We used the car dynamics in \eqref{eq:car_dynamics} and selected start and goal states training distribution (equivalently the feasible set for the ID results). We define $\mathcal{F} = \{(x,u) \mid p_x \in [0,5], p_y \in [-5,5], \Vert [p_x, p_y]^\top -[2.5, 0]\Vert_2 \ge 1\}$, 

\begin{wrapfigure}{r}{0.35\textwidth}
    \centering
    \vspace{-5pt}
    \includegraphics[width=\linewidth]{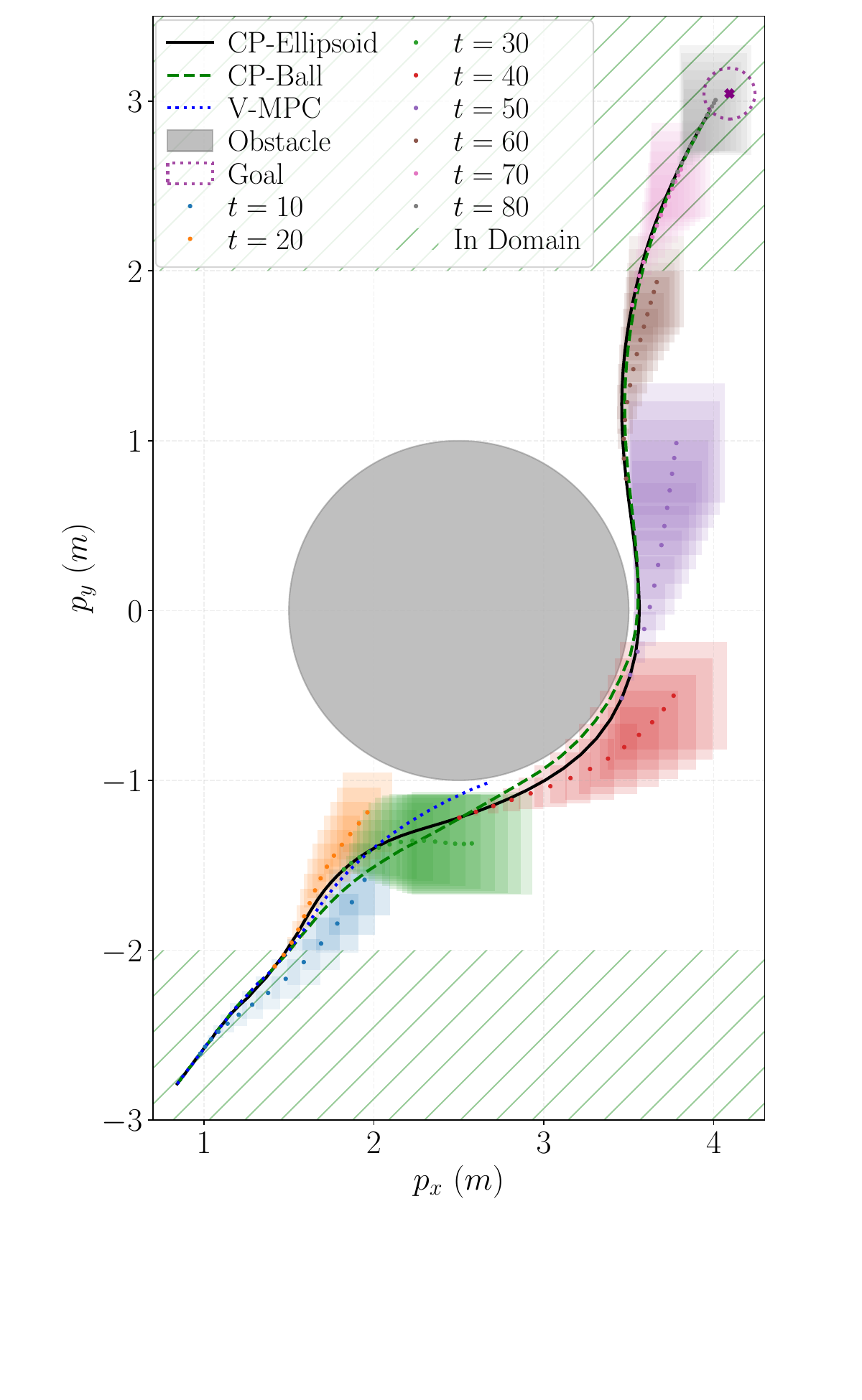}
    \vspace{-0.85in}
    \caption{\textbf{Friction Car.} We plot all approaches and forecasted MPC steps \& tubes for CP-Ellipsoid.}
    \label{fig:friction_car}
    \vspace{-10pt}
\end{wrapfigure}
\vspace{-5pt}

\noindent with bounds on $p_x$ and $p_y$ and a 1-m radius obstacle centered at $(2.5, 0)$. Across 10 runs, all three approaches (V-MPC, CP-Ball, and CP-Ellipsoid) safely reached the goal while avoiding the obstacle; however, V-MPC does not perform constraint tightening and lacks safety guarantees. In Table \ref{tab:combined_results_modified}, CP-Ball and CP-Ellipsoid have larger obstacle proximity--a byproduct of the constraint tightening (Table \ref{tab:combined_results_modified}).
Finally, the one-step prediction error for $\hat f$ remains inside the ball and ellipsoid over $99.33\% = 100(1-\frac{0.1}{15})\%$ during execution--validating the coverage guarantees (Fig. \ref{fig:ellip_in_domain} in App. \ref{app:additional_results}).
In the training domain, we expect V-MPC to have similar success rates as our method since $\hat f$ is accurate; this changes in OOD settings. We find that our approach avoids the obstacle ID, without a major slowdown from constraint tightening computations.

\vspace{-5pt}
\paragraph{Dubins': OOD.} To evaluate OOD performance, we train a dynamics and uncertainty model where  $p_y \in [-12, -6] \cup [6,12]$ and modify the dynamics to steer the car’s angle toward the obstacle for $p_y \in [-6,6]$ (see App. \ref{app:car_dyn}), while placing the feasible set of the rollout (same as the ID experiment) OOD. In this setting, we vary the calibration set size from 2,250 to 4,250 points (Fig. \ref{fig:OOD_tubes}) for CP-Ellipsoid, ensuring the start and goal are outside the training domain. 
\begin{wrapfigure}{r}{0.38\textwidth}
\vspace{-16pt}
    \centering
    \includegraphics[width=\linewidth]{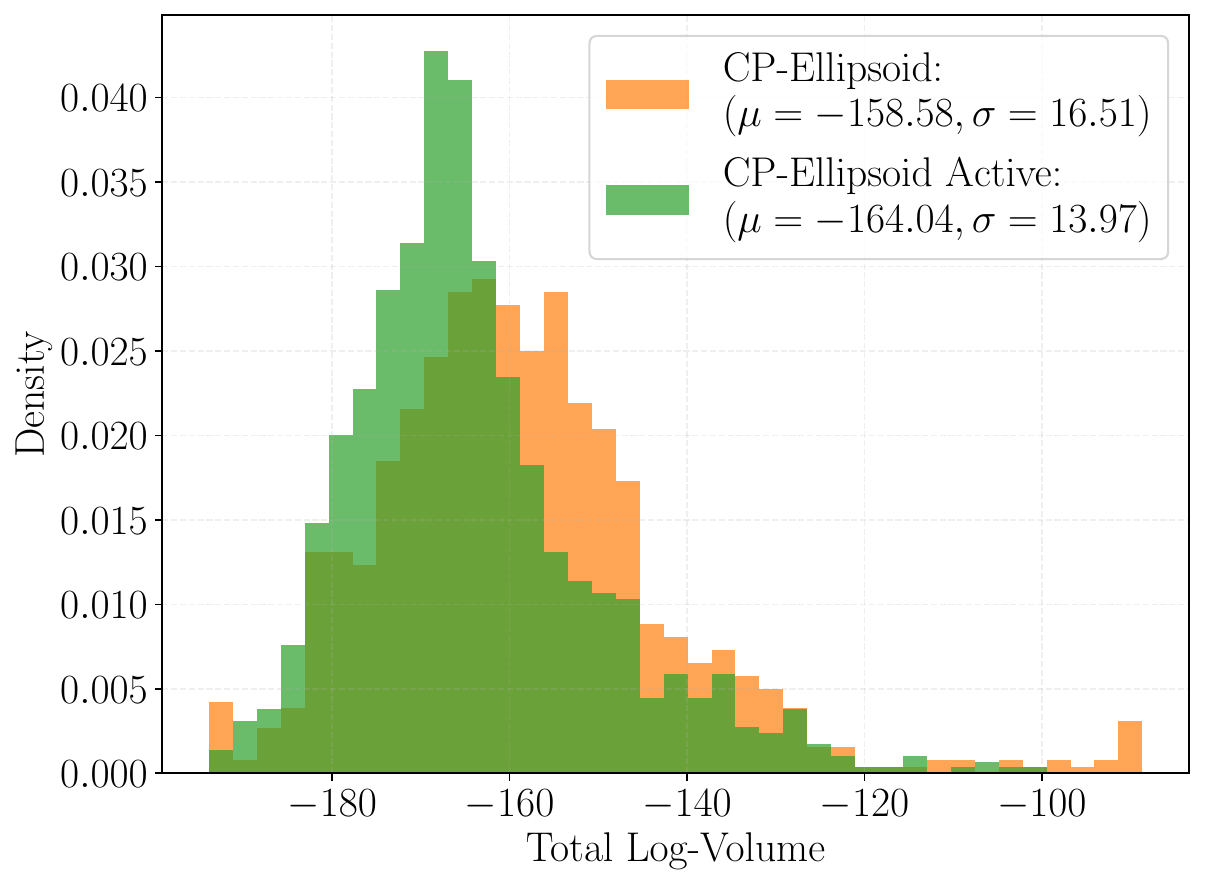}
    \vspace{-25pt}
    \caption{\textbf{Active Uncertainty.} We plot the histogram of summed log-volumes for each forecasted MPC plan. With active uncertainty reduction, the distribution shifts left, indicating less uncertainty at runtime.
    }
\vspace{-10pt}
    \label{fig:friction_car}
\end{wrapfigure}
As shown in Fig.~\ref{fig:OOD_tubes}, the one-step prediction errors initially lie outside the predicted ellipsoid but gradually return within it as online calibration points are collected. Moreover, with fewer initial calibration points, the online adaptation occurs more rapidly, resulting in a higher coverage of errors in the ellipsoid. The plot also shows that the one-step tubes in $\theta$ are wider for smaller $\calib$, indicating that data scarcity increases conservativeness. The goal is reliably reached, suggesting that the SLS tubes, when informed by CP, enable robust OOD performance.
\vspace{-5pt}
\paragraph{Dubins': Disjoint Training Domains (Friction Car).} 
To evaluate OOD obstacle avoidance under limited online data, we construct disjoint training regions (similar to the OOD setup) where  $p_y \in [-5, -2] \cup [2,5]$. The start and goal positions are randomly sampled from different regions, ensuring that all feasible trajectories must traverse OOD areas to reach the goal. In the OOD region, we apply the same steering modification as in the previous experiment, while adding $p_y$ dependent acceleration dynamics (see App.~\ref{app:car_dyn} for details).
Across 10 runs, V-MPC failed twice, the CP-Ball failed once, and CP-Ellipsoid was consistently successful. These differences arise from reduced obstacle clearance (Table~\ref{tab:combined_results_modified}) observed in the CP-Ball and V-MPC cases. Fig.~\ref{fig:friction_car} shows that both CP-Ball and CP-Ellipsoid reach the goal, but CP-Ellipsoid maintains a greater distance from the obstacle, as evident from its forecasted trajectories and uncertainty tubes. Fig. \ref{fig:friction_car} also illustrates that the state-dependent ellipsoids expand along the direction of motion, as increased forecasted $y$-motion leads to higher $y$-direction uncertainty. A result from an extra run is in App.~\ref{app:additional_results}.

\paragraph{Dubins': Active Uncertainty Reduction.}
Finally, to evaluate active uncertainty reduction, we trained $\hat{f}$ and $\covariance$ without sampling data in a circular region centered at $(p_x, p_y) = (2.5, 0.0)$ with a radius of 1 meter, and introduced friction and attractive steering dynamics within this unseen region. Thus the feasible set contains an OOD region, where rollouts are permitted to traverse. Despite this, the active uncertainty reduction method achieves lower prediction error by planning a longer path that remains in the training domain (Table~\ref{tab:active_uncert} App. \ref{app:additional_results}). Moreover, Fig.~\ref{fig:friction_car} shows the log-volume of the uncertainty tubes, showing that the active uncertainty cost (\eqref{eq:active_cost} in App.~\ref{app:uncert_reduc}) yields smaller tubes (Fig. \ref{fig:friction_car}, green) compared to the baseline without this cost (Fig. \ref{fig:friction_car}, orange).
\vspace{-4pt}
\begin{table}[h]
    \centering
  { 
  \footnotesize  
  \centering
  \caption{\looseness-1Mean and standard deviation of computation times and prediction errors and the mean minimum obstacle distance for vanilla-MPC (V-MPC), the CP-Ball (Ball), and CP-Ellipsoid (Ellipsoid). }\vspace{-10pt}
  \label{tab:combined_results_modified}
  \setlength{\tabcolsep}{2pt} 
  \begin{tabular}{@{}lccccccccc@{}}
    \toprule
    & \multicolumn{3}{c}{\textbf{Computation Time (ms)}} 
    & \multicolumn{3}{c}{\textbf{Prediction Error ($L_2$) ($\times 10^{-2}$)}} 
    & \multicolumn{3}{c}{\textbf{Min. Obstacle Dist. (m)}} \\
    \cmidrule(lr){2-4} \cmidrule(lr){5-7} \cmidrule(lr){8-10}
    \textbf{Model} 
    & \textbf{V-MPC} & \textbf{Ball} & \textbf{Ellipsoid} 
    & \textbf{V-MPC} & \textbf{Ball} & \textbf{Ellipsoid} 
    & \textbf{V-MPC} & \textbf{Ball} & \textbf{Ellipsoid} \\
    \midrule
    
    
    Car 
    & 53.3 $\pm$ 2.1 & 126.6 $\pm$ 8.4 & 126.7 $\pm$ 8.3 
    & 0.28 $\pm$ 0.54 & 0.26 $\pm$ 0.14 & 0.26 $\pm$ 0.14 
    & 1.152 & 1.202 & 1.207 \\
    
    Friction Car 
    & 53.7 $\pm$ 2.6 & 131.1 $\pm$ 1.4 & 129.7 $\pm$ 9.0 
    & 1.94 $\pm$ 2.87 & 1.92 $\pm$ 2.80 & 1.95 $\pm$ 2.69 
    & 1.315 & 1.376 & 1.464 \\
    
    Quadcopter 
    & 63.1 $\pm$ 4.6 & nan $\pm$ nan & 222 $\pm$ 20.9 
    & 3.89 $\pm$ 5.49 & nan $\pm$ nan & 3.30 $\pm$ 4.17 
    & 0.804 & nan & 0.809 \\
    \bottomrule
  \end{tabular}\vspace{-5pt}
  } 
\end{table}

\paragraph{12D Quadcopter}
\begin{SCfigure}
\vspace{-10pt}
    \centering
    \includegraphics[width=0.65\textwidth]{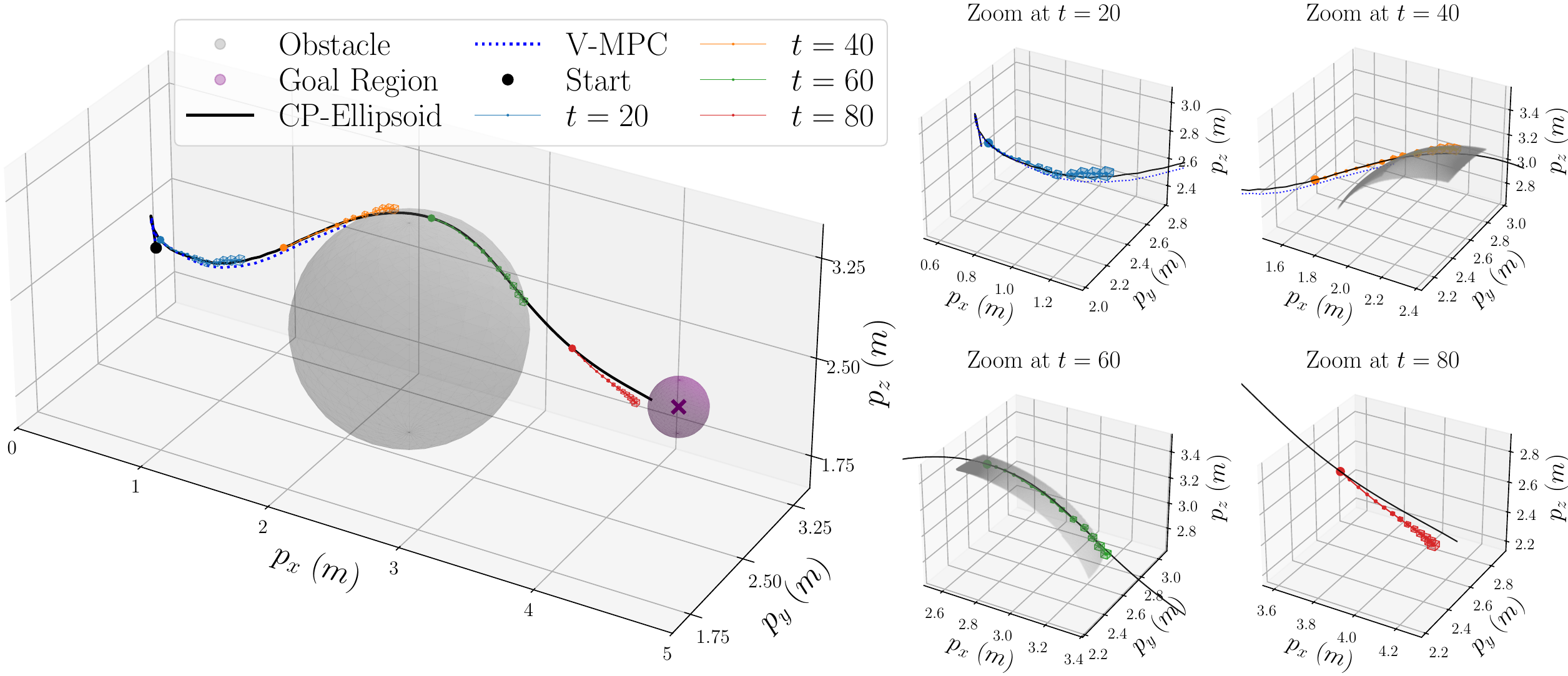}
    \vspace{-7pt}
    \caption{\textbf{Quadcopter.} A rollout of the quadcopter with V-MPC and CP-Ellipsoid. CP-Ellipsoid maintains sufficient proximity to avoid crashing into the obstacle, while V-MPC crashes.}
    \label{fig:quad_trajectory}
\end{SCfigure}
\vspace{-10pt}
To evaluate the scalability of our method to higher-dimensional systems, we train a 12D quadcopter dynamics model. Using this model, we plan trajectories such that the spatial coordinates $(p_x, p_y, p_z)$ remain outside a sphere centered at $(2.5, 2.5, 2.5)$ with a radius of 0.8 meters -- leaving a 0.2-meter-thick shell around the sphere that is out of distribution and characterized by increased downward acceleration. To improve the performance of the uncertainty model, we modify its weighting by blending the covariance matrix with the identity matrix to reduce conservativeness, i.e., using $((1-\tau) I + \tau \covariance(z,v))$ instead of $\covariance(z,v)$ in \eqref{eq:ellip_prediction_set}. With this approach, the CP-Ellipsoid method successfully planned 8 out of 10 trajectories (one failure was due to numerical instability, which did not result in collision), whereas V-MPC approach succeeded in only 4 out of 10 runs. While this demonstrates improved robustness over V-MPC, the failed case suggests faster tube expansion may be necessary to ensure consistent collision avoidance. Fig.~\ref{fig:quad_trajectory} shows an example quadcopter trajectory, demonstrating that our CP-informed tubes maintain safe obstacle proximity.

\vspace{-11pt}
\section{Conclusion}
\vspace{-4pt}

In this work, we developed an efficient framework for robust out-of-distribution MPC with learned dynamics, using WCP and SLS. We provided theoretical guarantees ensuring that the rollout satisfy constraints with high probability, and introduced an uncertainty reduction cost function to encourage the system to remain within the training domain when possible. Through experiments on variants of the 4D Dubins car and 12D quadcopter, we demonstrated that the \methodnospace{} framework achieves robust performance even when using learned dynamics that generalize poorly OOD.
\acks{We thank Devesh Nath for insightful discussions and feedback on this work.}

\bibliography{l4dc2026-sample}
\appendix

\newpage
\section{Proofs:}\label{app:proofs}
\subsection{Proof of Theorem \ref{thm:cov_gap_thm} and Corollary \ref{cor:cov_gap_infty}}
\begin{proof} \textbf{(Theorem \ref{thm:cov_gap_thm})}
From Theorem 2 and Lemma 1 in \citet{barber2023conformal}, the coverage gap for independent data is bounded by:
\begin{equation}\label{eq:init_barber}
\text{Coverage Gap} \le \sum_{i=1}^{N_\textrm{calib}} \tilde{w}_i \cdot d_{TV}(S^{i,k}, S^{k,k}) \le \sum_{i=1}^{N_\textrm{calib}} \tilde{w}_i \cdot 2 d_{TV}(S_{i,k}, S_{k,k})
\end{equation}
Given the normalized weights are $\tilde{w}_i = \frac{w_i}{1 + \sum_{j=1}^{N_\textrm{calib}} w_j}$, and $w_i = \rho^{d_i}$ for $d_i = ||[z_k\T,v_k\T] - [x_i\T,u_i\T]||_2$
We can substitute into Equation \ref{eq:init_barber}, and prove the \textit{Tight Bound} in Theorem \ref{thm:cov_gap_thm}:
\begin{equation}\label{eq:sub_barber}
\text{Coverage Gap} \le \sum_{i=1}^{N_\textrm{calib}} \left( \frac{\rho^{d_i}}{1 + \sum_{j=1}^{N_\textrm{calib}} \rho^{d_j}} \right) \cdot 2\epsilon \cdot d_i
\end{equation}
Without loss of generality, we can rearrange the sum to create a decaying series $\{\rho^{d_i}\}_{i=1}^{N_\textrm{calib}}$, where $d_i$ is increasing. While this series is not purely geometric, we can bound it by two geometric series. Let $d_\text{min} = \min\{1 \leq i \leq N_\textrm{calib}-1 | d_{i+1} - d_i\}$ and $d_\text{max} = \max\{1 \leq i \leq N_\textrm{calib}-1 | d_{i+1} - d_i\}$. Then, $\rho^{d_1}\rho^{d_\text{max}(i-1)} \leq \rho^{d_i}  \leq \rho^{d_1}\rho^{d_\text{min}(i-1)}$. Thus, we can bound Equation \ref{eq:sub_barber}, 
$$
\text{Coverage Gap} \le \sum_{i=1}^{N_\textrm{calib}} \left( \frac{\rho^{d_i}}{1 + \sum_{j=1}^{N_\textrm{calib}} \rho^{d_j}} \right) \cdot 2\epsilon \cdot d_i \leq 2\epsilon \left( \frac{\sum_{i=1}^{N_\textrm{calib}}\rho^{d_1}p^{d_\text{min}(i-1)}\cdot (d_1 + d_\text{max}(i-1))}{1 + \sum_{j=1}^{N_\textrm{calib}} \rho^{d_1}\rho^{d_\text{max}(j-1)}} \right) 
$$
Then observe the sum of the numerator, 
\begin{align}
N_\text{sum} &= \sum_{i=1}^{N_\textrm{calib}}\rho^{d_1}p^{d_\text{min}(i-1)}\cdot (d_1 + d_\text{max}(i-1)) \\
&=\rho^{d_1} \left[ d_1 \left( \frac{1 - \rho^{N_\textrm{calib} \cdot d_\text{min}}}{1 - \rho^{d_\text{min}}} \right) + d_\text{max} \left( \frac{\rho^{d_\text{min}} - N_\textrm{calib} \rho^{n \cdot d_\text{min}} + (N_\textrm{calib}-1)\rho^{(N_\textrm{calib}+1)d_\text{min}}}{(1-\rho^{d_\text{min}})^2} \right) \right]
\end{align}

Similarly, the sum the of the denominator,
\begin{align}
    D_\text{sum} &= 1 + \sum_{j=1}^{N_\textrm{calib}} \rho^{d_1}\rho^{d_\text{max}(j-1)}\\
    &= 1 + \rho^{d_1} \left( \frac{1 - \rho^{N_\textrm{calib} \cdot d_\text{max}}}{1 - \rho^{d_\text{max}}} \right)
\end{align}

Then, the coverage gap is bounded by,
\begin{small}
\begin{align}
    &\text{Coverage Gap} \le 2\epsilon\frac{N_\text{sum}}{D_\text{sum}} \\
    &\le 2 \epsilon \left[{d_1 \left( \frac{1 - \rho^{N_\textrm{calib} \cdot d_\text{min}}}{1 - \rho^{d_\text{min}}} \right) +
    d_\text{max} \left( \frac{\rho^{d_\text{min}} \overbrace{- N_\textrm{calib} \rho^{N_\textrm{calib} \cdot d_\text{min}} + (N_\textrm{calib}-1)\rho^{(N_\textrm{calib}+1)d_\text{min}}}^{\leq 0}}{(1-\rho^{d_\text{min}})^2} \right)}\right]\underbrace{\frac{1 - \rho^{d_\text{max}}}{1 - \rho^{N_\textrm{calib} \cdot d_\text{max}}}}_{\leq 1}\label{eq:n_sum_d_sum}\\
    &\le 2 \epsilon \left[  \frac{d_1 }{1 - \rho^{d_\text{min}}} + \frac{d_\text{max}\rho^{d_\text{min}}}{(1-\rho^{d_\text{min}})^2}\right].\label{eq:final_bound} 
\end{align}
Hence, we have proved the \textit{Interpretable Bound}.
\end{small}
\end{proof}

For large $N_\textrm{calib}$, (i.e., as $N_\textrm{calib}\to\infty$), we can achieve a tighter upper bound
\begin{corollary}\label{cor:cov_gap_infty} Under the same assumptions of Theorem \ref{thm:cov_gap_thm} we have that as $N_\textrm{calib}\to \infty$ the coverage gap is more tightly bounded by,
\begin{equation}\text{Coverage Gap}\le 2\epsilon
\left[ \frac{d_1 }{1 - \rho^{d_\text{min}}} + \frac{d_\text{max}\rho^{d_\text{min}}}{(1-\rho^{d_\text{min}})^2}\right](1-\rho^{d_\text{max}})
\end{equation}
\end{corollary}

\begin{proof}\textbf{(Corollary \ref{cor:cov_gap_infty})}
Taking Equation \ref{eq:n_sum_d_sum} as $n\to\infty$ we complete the proof. 
\end{proof}

\subsection{Proof of Theorem \ref{thm:tube_gap}}
\begin{proof}\textbf{(Theorem \ref{thm:tube_gap})}
    Using Equation \ref{eq:init_barber}, we have that 
    \begin{equation}\label{eq:init_barber_2}
    \text{Coverage Gap} \le \sum_{i=1}^{N_\textrm{calib}} \tilde{w}_i \cdot d_{TV}(\tilde{S}^{i,k}, \tilde{S}^{k,k}) \le \sum_{i=1}^{N_\textrm{calib}} \tilde{w}_i \cdot 2 d_{TV}(S_{i,k}, \tilde{S}_{k,k}),
    \end{equation}
    where $(s_{1,k}, s_{2,k}, s_{i-1,k}, \tilde{s}_{k,k}, s_{i+1,k}, \ldots, s_{N_\textrm{calib},k}) \sim \tilde{S}^{i,k}$, $(s_{1,k}, \ldots, s_{N_\textrm{calib},k}, \tilde{s}_{k,k}) \sim \tilde{S}^{k,k}$, $\tilde{s}_{k,k} = s_k(x_k,u_k,f(x_k,u_k))$ is unknown due to uncertainty in $f(x_k,u_k)$, and $\tilde{s}_{k,k} \sim \tilde{S}_{k,k}$ is the distribution of the non-conformity score corresponding to the observed state and control $(x_k,u_k)$. Then, using the triangle inequality we have that,
    \begin{equation}\label{eq:split_into_sums}
    \text{Coverage Gap} \le \sum_{i=1}^{N_\textrm{calib}} \tilde{w}_i \cdot d_{TV}(\tilde{S}^{i,k}, \tilde{S}^{k,k}) \le \underbrace{\sum_{i=1}^{N_\textrm{calib}} \tilde{w}_i \cdot 2 d_{TV}(S_{i,k}, S_{k,k})}_{(*), \text{Bounded in Eq. \ref{eq:final_bound}}} + \underbrace{\sum_{i=1}^{N_\textrm{calib}} \tilde{w}_i \cdot 2 d_{TV}(S_{k,k}, \tilde{S}_{k,k})}_{(**)}.
    \end{equation}
    Observing that the left sum $(*)$ is already bounded by Theorem \ref{thm:cov_gap_thm}, we proceed to bound the right sum $(**)$. Then, realizing the sum of normalized weights is $\le 1$, we obtain, 
    \begin{equation}\label{eq:near_by_cov}
        \sum_{i=1}^{N_\textrm{calib}} \tilde{w}_i \cdot 2 d_{TV}(S_{k,k}, \tilde{S}_{k,k}) = 2 d_{TV}(S_{k,k}, \tilde{S}_{k,k})\sum_{i=1}^{N_\textrm{calib}} \tilde{w}_i \le 2 d_{TV}(S_{k,k}, \tilde{S}_{k,k})\
    \end{equation}
     Then, using the Lipschitz-type bound, 
    \begin{equation}
        2 d_{TV}(S_{k,k}, \tilde{S}_{k,k}) \le 2\hat{\epsilon}||[x_k\T,u_k\T]-[z_k\T,v_k\T]||_2 \le 2\hat{\epsilon}M(\mathcal{R}_k) =: \gamma({\mathcal{R}_k}),
    \end{equation} 
    where, $(x_k,u_k) \in \mathcal{R}_k$, $M(\mathcal{R}_k)$ is the maximum axis length of the tube, and $\hat{\epsilon} \le \epsilon$ is the local Lipschitz constant in the tube. The remaining proof follows substituting the bound, $\gamma({\mathcal{R}}_k)$ for $(**)$ in \eqref{eq:split_into_sums}.
\end{proof}

\subsection{Proof Sketch of Theorem \ref{thm:total_coverage} and Theorem \ref{cor:rollout_prob}}
\begin{proofsketch} \textbf{(Theorem \ref{thm:total_coverage})}
Using Theorem \ref{thm:tube_gap} we get the probability of error coverage for each step in the rollout. Using the approach of Proposition 2 and Theorem 1 from \citet{DBLP:conf/l4dc/CheeSHP24} and \citet{DBLP:journals/ral/LindemannCSP23} respectively we complete the proof. 
\end{proofsketch}
\begin{proofsketch}\textbf{(Theorem \ref{cor:rollout_prob})}
\textbf{(a)} Since SLS provides guarantees conditioned on valid disturbance bounds, Theorem \ref{thm:total_coverage} directly implies the result in Corollary \ref{cor:rollout_prob}(a).

\textbf{(b)} Instead of applying the union bound approach across the forecasted MPC plan, we bound the intersection probability of the first-step errors remaining in the predicted ellipsoids for all $R$ planning steps. Then, using the same approach as Theorem \ref{thm:tube_gap} we get: 
\begin{equation}
        \Pr[\bigcap_{k=1}^{T}~\varepsilon(z_1^{(q)}, v_1^{(q)}) \in \mathcal{E}(z_1^{(q)}, v_1^{(q)})] \geq 1- \sum_{q=1}^{R}\sigma_1^q,
\end{equation}
Notably, we omit the tubes since the tubes at the first step of each replan, $\mathcal{R}_1^{q} = \{(z_1^{(q)}, v_1^{(q)})\}$ since the planner knows the initial state. Then, using the same approach as \textbf{(a)} we get the result in \textbf{(b)}.  
\end{proofsketch}

\section{Dynamics Models}
Below we discuss the dynamics models used in our experiments. We provide the continuous time dynamics which were used to train our models. We obtained the discrete model by using Euler Discretization. 
\subsection{Car Dynamics Models}\label{app:car_dyn}
\begin{equation}\label{eq:car_dynamics}
\begin{bmatrix}
    \dot{p}_x \\
    \dot{p}_y \\
    \dot{\theta} \\
    \dot{v}
\end{bmatrix}
=
\begin{bmatrix}
    v \cos(\theta) \\
    v \sin(\theta) \\
    0 \\
    0
\end{bmatrix}
+
\begin{bmatrix}
    0 & 0 \\
    0 & 0 \\
    1 & 0 \\
    0 & 1
\end{bmatrix}
\begin{bmatrix}
    \omega \\
    a
\end{bmatrix}
\end{equation}
Equation \ref{eq:car_dynamics} describes the control-affine dynamics models for Dubins' car, where $p_x$ and $p_y$ are the car's x and y positions. $\theta$ and $v$ are the direction and velocity of the car. Lastly, the control inputs directly steer ($\omega$) and accelerate ($a$) the car. The dynamics model in  \eqref{eq:car_dynamics} was used strictly for the in-distribution experiments. 

\paragraph{OOD Experiment Dynamics} For the out of distribution (OOD) Dubins' car experiment we modified the $\theta$ dynamics by adding attractive steering towards the obstacle centered at (2.5, 0) in the region without training data,
\begin{equation}\label{eq:theta_mod}
    \dot\theta = \omega_\text{attr}\1[|p_y| < 6] + \omega 
\end{equation}

\begin{subequations}\label{eq:w_attr}
\begin{align}
    d^2 &= (p_x - 2.5)^2 + p_y^2 \\
    \theta_{a} &= \text{atan2}(p_y,p_x) \\
    \Delta\theta_{norm} &= \text{atan2}(\sin(\theta_{a} - \theta), \cos(\theta_{a} - \theta)) \\
    \omega_{attr} &= \left( \frac{k_\text{attr}}{d^2} \right) \Delta\theta_{norm},
\end{align}
\end{subequations}
where the attractive steering dynamics are defined in Equation \ref{eq:w_attr}, and $\text{atan2}$ is the arctangent function accounting for the x-y quadrant, and $\1[\cdot]$ is the indicator function. For the OOD experiment we set $k_\text{attr}$, the strength of the attractive force to -0.5 (positive values repel). Notably, the strength of the force increases as the car gets closer to the obstacle. 

\paragraph{Active Uncertainty Experiment Dynamics} Similar to the OOD dynamics we include attractive steering (with $k_\text{attr} = -0.5$, we modify the the $v$ dynamics of the car such that, 
\begin{subequations}
    \begin{align}
    &\dot\theta = \tilde{\omega}_\text{attr}\1[d^2 < 1] + \omega \\
    &\dot v =  -(1 + \cos(\pi \sqrt{d^2}))\1[d^2 < 1]+ a \\
    &\tilde\omega_{attr} = \left( \frac{k_\text{attr}}{d^2 + 0.1} \right) \Delta\theta_{norm},
    \end{align}
\end{subequations}
where $d^2$ and $\Delta\theta_{norm}$ are defined in Equation \ref{eq:w_attr}. We modify the denominator of $\omega_\text{attr}$ to prevent a singularity at (2.5, 0.0)--this was not a problem in the OOD case since there was an obstacle present at (2.5, 0.0). The augmented velocity dynamics add a friction term to reduce the speed of the car in this region.

\paragraph{Friction Car Dynamics} Lastly, we modified Dubins' car such that the velocity dynamics were modified over the whole state-space, while including attractive steering (with $k_\text{attr}=-1.5$): 
\begin{subequations}
    \begin{align}
    &\dot\theta = {\omega}_\text{attr}\1[|p_y| < 2] + \omega \\
    &\dot v =  a-\text{sign}(v)\left(0.1\cos(\frac{2\pi}{5}p_y)) + 0.1 - a_\text{slip}\1[|p_y| < 2]\right)\\
    &a_\text{slip} = 0.7e^{-2p_y^2} 
    \end{align}
\end{subequations}
In this model, we include a slippage term to increase the velocity of the car for $-2 < p_y < 2$, while providing a changing friction term over the whole state-space. 

\subsection{12D Quadcopter Dynamics Model}\label{app:quad_dyn}
\begin{equation}
\begin{bmatrix}
    \dot{p}_x \\
    \dot{p}_y \\
    \dot{p}_z \\
    \dot{\psi} \\
    \dot{\theta} \\
    \dot{\phi} \\
    \ddot{p}_x \\
    \ddot{p}_y \\
    \ddot{p}_z \\
    \dot{p} \\
    \dot{q} \\
    \dot{r}
\end{bmatrix}
= f(x,u) :=
\begin{bmatrix}
    \dot{p}_x \\
    \dot{p}_y \\
    \dot{p}_z \\
    q \sin(\phi)/\cos(\theta) + r \cos(\phi)/\cos(\theta) \\
    q \cos(\phi) - r \sin(\phi) \\
    p + q \sin(\phi)\tan(\theta) + r \cos(\phi)\tan(\theta) \\
    \frac{u_1}{m} (\sin(\phi)\sin(\psi) + \cos(\phi)\cos(\psi)\sin(\theta)) \\
    \frac{u_1}{m} (\cos(\psi)\sin(\phi) - \cos(\phi)\sin(\psi)\sin(\theta)) \\
    g + \frac{u_1}{m} (\cos(\phi)\cos(\theta)) -a_\text{fall}\\
    \frac{I_y - I_z}{I_x} qr + \frac{u_2}{I_x} \\
    \frac{I_z - I_x}{I_y} pr + \frac{u_3}{I_y} \\
    \frac{I_x - I_y}{I_z} pq + \frac{u_4}{I_z}
\end{bmatrix}
\end{equation}
\begin{align}
    a_\text{fall} &= \left(0.1 + 0.1\cos(2\pi d^2 -\pi)\right)\1[{d^2 < 1}]\\
    d^2 &= (p_x-2.5)^2 + (p_y-2.5)^2 + (p_z-2.5)^2
\end{align}

For our experiments, we used the 12D Quadcopter dynamics model described in Equation \ref{app:quad_dyn} \citep{sabatino2015quadrotor}. $p_x, p_y, \text{ and } p_z$ describe the x, y, and z positions of the quadcopter while $\dot p_x, \dot p_y, \text{ and } \dot p_z$ describe the x, y, and z velocities. ${\psi}, {\theta}, \text{ and } {\phi}$ are the roll, pitch, and yaw. Similarly, $p, q, \text{ and } r$ are the roll rate, pitch rate, and yaw rates, respectively are directly controlled by inputs $u_2, u_3, \text{ and } u_4$. The control input, $u_1$, directly impacts the thrust--z-direction acceleration. Lastly, the constants $g$, $m$, $I_x$, $I_y$, and $I_z$ are the acceleration due to gravity, the mass of the quadcopter, and the principle moments of inertia. For our experiments, we let $g=-9.81 m/s^{2}$, $m = 1.0 kg$, $I_x = 0.5 kg\cdot m^2$, $I_y = 0.1 kg\cdot m^2$, and $I_z = 0.3 kg\cdot m^2$. Notably, the expression for $\ddot{p}_z$ includes $a_\text{fall}$ to introduce a falling term for OOD dynamics. The falling dynamics are not encountered while training the model as described in App. \ref{app:exp_data}. 

\section{Baseline Problem Definitions}\label{app:probs}

We define the baselines used in Sec. \ref{sec:results}. First, we define Vanilla MPC (V-MPC) \eqref{eq:vanilla_mpc}, which is a nominal na\"ive MPC planning under the learned dynamics without any constraint tightenings or model error calibration:
\begin{subequations}\small
\label{eq:vanilla_mpc}
\begin{align}
\min_{\substack{\mathbf{z}, \mathbf{v}}}\quad &  J(\mathbf{z},\mathbf{v}) + J_{\mathcal{X}_\textrm{f}}(\Z,\V)\\[-6pt]
\text { s.t. }\quad\quad &  {z}_{k+1}= \hat{f}(z_k,v_k),~ {z}_1 = \bar x_0,\quad\quad \forall k = 1,\ldots, T,\\
&g_{i}(z_k,v_k) + b_{i}  \leq 0,\\
&  \quad\quad \forall i =1, \ldots, n_c,\quad \forall k= 1,\ldots,T. \nonumber
\end{align}
\end{subequations}

We also define CP-Ball \eqref{eq:cp_ball}, which is a variant of our method which computes CP-calibrated model error bounds, except using a spherical, constant error bound representation, which is more conservative than our ellipsoidal bound:
\begin{subequations}\small
\label{eq:cp_ball}
\begin{align}
\min_{\substack{\mathbf{\Phi}_{\mathbf{x}},\mathbf{\Phi}_{\mathrm{u}},\mathbf{z}, \mathbf{v}}}\quad &  J(\mathbf{z},\mathbf{v})  + J_{\mathcal{X}_\textrm{f}}(\Z,\V) +  \tilde H_0(\Px, \Pu),\\[-6pt]
\text { s.t. }\quad\quad &  {z}_{k+1}= \hat{f}(z_k,v_k),~ {z}_1 = \bar x_0,\quad\quad \forall k = 1,\ldots, T,\\
&\Px_{k+1,j} =\bm{\nabla}_x \hat{f}\left(z_k, v_k\right)  \Px_{k,j} + \bm{\nabla}_u \hat{f}\left(z_k, v_k\right) \Pu_{k,j},\ \forall j = 1,\ldots, T, \forall k = j+1,\ldots,T,\\
&\Px_{j+1,j} = V,\quad \forall j = 1,\ldots, T,\\
& \textstyle\sum_{j=1}^{k}\big\Vert[\nabla_xg_{i}(z_k,v_k)\Px_{k,j}, \nabla_ug_{i}(z_k,v_k)\Pu_{k,j}]\big\Vert_{2} + g_{i}(z_k,v_k) + b_{i} + g_i^\textrm{lin}(z_k,v_k)  \leq 0,\\
&  \quad\quad \forall i =1, \ldots, n_c,\quad \forall k= 1,\ldots,T. \nonumber
\end{align}
\end{subequations}

\noindent\textbf{Surrogate Cost:} All SLS formulations use \eqref{eq:regulizer} to minimize the volume of the tubes--a surrogate for uncertainty, where $\Px$ and $\Pu$ collects all $\Px_{k,j}$ and $\Pu_{k,j}$ respectively, and $P^{\frac{1}{2}}, Q^{\frac{1}{2}} \in \mathbb{R}^{n_x \times n_x}, R^{\frac{1}{2}}\in\mathbb{R}^{n_u \times n_u}$ are symmetric positive definite matrices.

\begin{equation}  \label{eq:regulizer}\small 
\tilde H_0(\Px,\Pu)\defeq 
\textstyle\sum_{j=0}^{T-1} \left( \|\Qf^{\frac{1}{2}} \Px_{T,j}\|_\F^2  + \sum_{k=j}^{T-1}  \left(\| Q^{\frac{1}{2}} \Px_{k,j}\|_\F^2+ \| R^{\frac{1}{2}} \Pu_{k,j}\|_\F^2 \right)\right),
\end{equation}

\section{Experiment Setup}\label{app:training}
\subsection{Data Sampling}\label{app:exp_data}
\paragraph{Dubins' Car \& Friction Car:} To train the in-distribution Dubins' car dynamics and uncertainty models, we applied uniform sampling to the following state-control space:
\begin{equation}\label{eq:in_dist_car}[p_x, p_y, \theta, v] \times[\omega, a]\in [0, 5] \times [-5, 5] \times [-\pi, \pi] \times [-10, 10] \times [-10, 10] \times [-10, 10]. \end{equation}
Similarly, for the out-of-distribution (OOD) car, we instead sampled $p_y \in [-12, -6] \cup [6,12]$ with the remaining state and controls remaining the same as \eqref{eq:in_dist_car}. For the friction car model, we sampled $p_y \in [-5, -2] \cup [2,5]$, while using \eqref{eq:in_dist_car} for the remaining dimensions. Lastly, for our active uncertainty experiments, we also sampled from Equation \ref{eq:in_dist_car}, and enforced that $(p_x - 2.5)^2 + (p_y)^2 > 1$ to create an OOD region. For both the OOD car and active uncertainty experiments, we ensured our learned model was restricted to training on areas where modifications did not impact the car dynamics model. 

Using the sampled points, we computed the ground truth using the expert dynamics model, as formulated in Equation \ref{eq:car_dynamics}, to create our dynamics training dataset (1,000,000 points), uncertainty training dataset (1,000,000 points), and calibration datasets (varying numbers of points).

\paragraph{12D Quadcopter:} To get informative data points to train our neural-network model, we collected expert trajectories such that the states satisfied,
\begin{subequations}\label{eq:opt_expert_quad}
\begin{align}
p_x, p_y, p_z &\in [0,5]\text{ s.t. } (p_x-2.5)^2 + (p_y-2.5) + (p_z-2.5)^2 > 1 \\
    \psi, \theta, \phi &\in [-\pi/4, \pi/4]\\
\dot{p}_x, \dot{p}_y, \dot{p}_z  &\in [-5, 5]\\
p,q,r &\in [-2, 2]\\
u_1 &\in [0, -2gm],~g = -9.81 m/{s^2}, m=1.0 kg \\
u_2, u_3, u_4 &\in [-0.1, 0.1]
\end{align}   
\end{subequations}
to create an OOD region centered at (2.5, 2.5, 2.5) with a radius of 1. Additionally, the goal position was $[p_x = 4.5, p_y = 2.5, p_z = 2.5]$, while the start positions was arbitrarily sampled from Equation \ref{eq:opt_expert_quad}. Our discrete timestep was $0.05s$ and our trajectory length was 100 timesteps. We sampled 100,000 trajectories and included state-control pairs from trajectories that the optimizer (IpOpt) was able to solve. This resulted in a dynamic training and uncertainty-trained set with $4,702,400$ points each, leaving up to $100,000$ points for the calibration dataset.

\subsection{Model Architecture and Training}
For our learned model (dynamics and uncertainty), we use an MLP with one hidden layer and the tanh activation function. Our uncertainty model was trained to output Cholesky factors, ensuring the matrices were positive-definite and facilitating faster performance in our optimization problem. 

For Dubins' car and the friction car dynamics models, we used $4,096$ and $2,048$ hidden nodes for the dynamics and uncertainty models, respectively.  
Similarly, for the 12D quadcopter, we used $1,024$ hidden nodes for both the dynamics and uncertainty models, and were trained for 200 epochs. Lastly, we used a learning rate of $10^{-4}$ and $10^{-5}$ for the dynamics and uncertainty models, respectively.

\subsection{Optimization Parameters and Cost Functions}\label{app:lqr_app}
For the SLS forward propagation and constraint tightenings, we used $\bar{P}^{\frac{1}{2}} = 10^6I_{n_x}$, $\bar{Q}^{\frac{1}{2}} = 10^6I_{n_x}$ and $\bar{R}^{\frac{1}{2}} = 10^6I_{n_u}$, as found in Equation \ref{eq:regulizer}. For trajectory optimization, we considered the following LQR cost, 
\begin{equation}\label{eq:lqr}
J(\Z,\V) :=  
      \sum_{k=1}^{T-1} \left( z_k\T Q  z_k + (z_{k+1} - z_k)\T Q_{s}  (z_{k+1} - z_k) + u_k\T R {u}_k \right),
\end{equation}
\begin{equation}\label{eq:lqr_terminal}
    J_{\mathcal{X}_\textrm{f}}(\Z,\V) := (z_T - z_\text{goal})\T Q_f (z_T - z_\text{goal})
\end{equation}
where $Q_f, Q_s, Q \in \mathbb{R}^{n_x \times n_x}$ are positive-semi-definite matrices, $R\in\mathbb{R}^{n_u \times n_u}$ is a positive-definite matrix, and $z_\text{goal}$ is the goal state. For all dynamics models, we used diagonal matrices, represented as $\text{diag}(a_1,a_2, \dots, a_n)$ and $\text{block\_diag}(\dots)$ representing block diagonal matrices. For Dubins' car (excluding the active uncertainty experiment) and the friction car, we let $Q_f = \text{diag}(1,1,0,1)$, $Q = Q_s = 0_{4 \times 4}$ (the $4 \times 4$ zero matrix), and $R = \text{diag}(0.1, 0.1)$. For the Quadcopter, $Q_f = \text{block\_diag}(I_3, 0_{3\times3}, I_3, 0.2I_3)$, $Q = \text{block\_diag}(0_{3\times3}, 0_{3\times3}, 0.01I_3,0_{3\times3})$, $Q_s = 0_{12\times12}$, and $R = \text{diag}(0.01, 0.1, 0.1, 0.1)$. 

Lastly, regarding conformal prediction parameters, we used $\rho = 0.97$ for the car model and $\rho = 0.8$ for the quadcopter model, respectively. Outside of the OOD car experiment, we used $10,000$ calibration points for the car models and $30,000$ calibration points for the quadcopter model. For all the models we set $\alpha_k = \frac{0.1}{15}$ s.t. the sum total of $\alpha$ would be 0.1.

\subsection{Active Uncertainty}\label{app:uncert_reduc} To enable active uncertainty reduction, we sought to minimize the distance to the positional-state variables (i.e., $p_x$ and $p_y$) from $\calib$. To achieve efficient performance, we used K-means clustering \cite{hastie2009elements} to find a representative set of $L$ points, $\tilde{\mathcal{D}} = \{({p_x}_i, {p_y}_i)\}_{i=1}^{L}$. Then, we minimized the following cost for Dubins' car,
\begin{align}
    &J_\text{active}(\Z,\V)\nonumber\\
    &=\sum_{k=0}^{T-1} \exp\left[-\frac{1}{L + G}\left(Ge^{-\beta||({p_x}_\text{goal}, {p_y}_\text{goal}, v_\text{goal}) - ({p_x}_k, {p_y}_k, v_k)||^2} + \sum_{j=1}^{L}e^{-\beta||({p_x}_j, {p_y}_j) - ({p_x}_k, {p_y}_k)||^2}\right)\right], \label{eq:active_cost}
\end{align}
where ${p_x}_\text{goal}, {p_y}_\text{goal},\text{ and } v_\text{goal}$ are the goal x-position, y-position, and velocities, ${p_x}_k, {p_y}_k, \text{ and }v_k$ are the x-position, y-position, and velocities of the nominal point $z_k$ while ${p_x}_i, \text{ and }{p_y}_i$ are the x and y positions of points from $\tilde{\mathcal{D}}$, and $G$ is a weight for the goal node (effectively creating G points at the goal). Lastly, $\beta$ is a sharpness parameter to ensure that we minimize the distance to a particular point in $\tilde{\mathcal D}$ as opposed to minimizing the distance to all points. The inclusion of the goal point enabled us to provide a high weight to the cost without hinder the MPC optimizer from reaching the goal. Since our goal is to minimize the distance to the representative points by minimizing the cost function we scale the internal sum by $-1$. For efficient implementation we included this as a part of the quadratic cost, thus requiring the outer-exponential to provide a lower-bound to the cost. In our experiments we choose $L=800$, $G=200$, and $\beta = 3$, and our total cost was $J_\text{total}(\Z,\V) = J(\Z,\V) + J_{\mathcal{X}_\textrm{f}}(\Z,\V) + 1000J_\text{active}(\Z,\V)$. For $Q_f, Q, \text{ and } R$ we used the same parameters as the remaining Dubins' car experiments and let $Q_s = 0.1I_4$. 

\section{Theorem \ref{thm:cov_gap_thm} Analysis}\label{app:theoretical_expansion}

\begin{figure}[htbp]
    \centering
    \includegraphics[width=\textwidth]{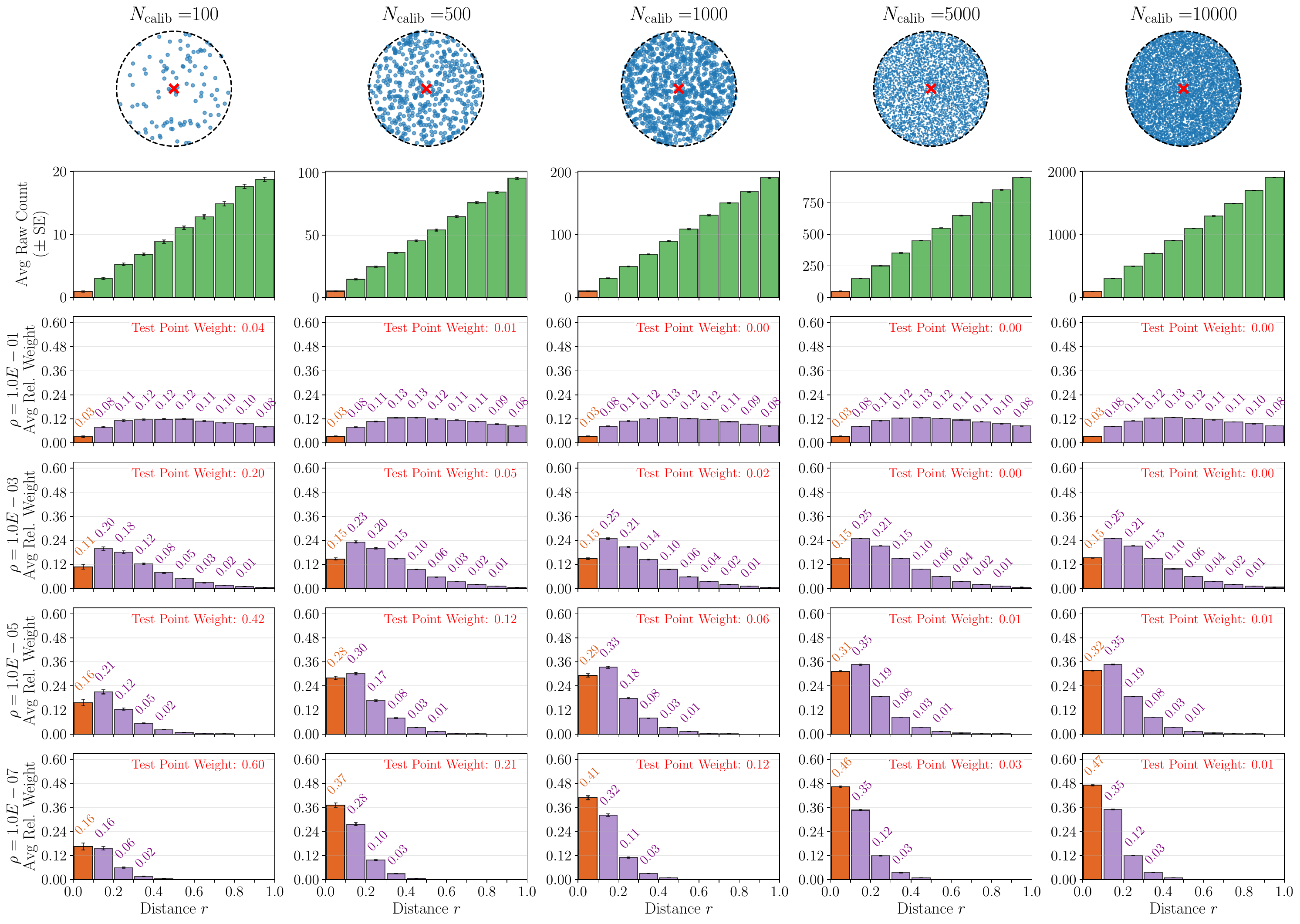}
    \caption{\textbf{Toy Example.} We plot the effect of the calibration set size, $N_\textrm{calib}$, and compare the weightage provided to points using the exponential weight decay function at different $\rho$ values. The results are averaged across 100 trials of randomly sampling $N_\textrm{calib}$ points in the unit circle, with an example of the calibration data location in the top row. The second row (a histogram) shows that, under uniform sampling on the unit circle, the number of points at each radius (Distance r) increases linearly. The remaining rows provide the weight contribution of each histogram bin by normalizing the weights (as discussed in Section~\ref{sec:cp_err_bound}) and summing the normalized weights for each bin. We observe that smaller $\rho$ values place more emphasis on points closer to the test point, and the weight of the test point, represented by a point mass at $\infty$ in the empirical score distribution, increases. Lastly, in the plot, red denotes the test point, while orange highlights the bin of points closest to it.}
    \label{fig:point_distribution}
\end{figure}

For completeness we restate the \textit{Coverage Gap} bound in Theorem \ref{thm:cov_gap_thm}: 
$$\text{Coverage Gap}\le \underbrace{\sum_{i=1}^{N_\textrm{calib}} \left( \frac{\rho^{d_i}}{1 + \sum_{j=1}^{N_\textrm{calib}} \rho^{d_j}} \right) \cdot 2\epsilon \cdot d_i}_{\text{Tight Bound}} \le \underbrace{2 \epsilon \left[  \frac{d_1 }{1 - \rho^{d_\text{min}}} + \frac{d_\text{max}\rho^{d_\text{min}}}{(1-\rho^{d_\text{min}})^2}\right]}_{\text{Interpretable Bound}}$$
As discussed in Section \ref{sec:theory}, an increase in the density of state-control calibration points reduces $d_\textrm{min}$, thereby increasing the (interpretable) coverage gap bound in \eqref{eq:thm3_bound}. We can reduce the bound by reducing $\rho$. To build intuition, on the results, we consider the following toy example: we have calibration points uniformly sampled from a unit circle with a test point located at the origin, $(0,0)$, such that the distribution of the non-conformity score is spatially varying (i.e. $\tvd \neq 0$). Without defining the distribution, Figure \ref{fig:point_distribution} demonstrates that at smaller $\rho$ values, greater emphasis is provided to non-conformity scores located near the test point, with smaller TV distances to the test point. Hence, the coverage gap decreases for a fixed $N_\textrm{calib}$ as $\rho$ decreases. The interpretable bound also implies that when $d_\textrm{min}$ is reduced, the coverage gap increases, which can be compensated for by requiring $\rho$ to decrease. Figure \ref{fig:point_distribution}, illustrates that the coverage gap increases at larger $d_\textrm{min}$ stemming from increased data density and $N_\textrm{calib}$, due to a higher weightage being provided to the test point, represented as a point-mass at $\infty$ in the empirical weighted non-conformity score distribution, resulting in a larger $1-\alpha$ quantile value. The larger quantile value then increases coverage, thereby reducing the coverage gap.

To understand the tightness of the derived bound, we use a gamma distribution to randomly sample non-conformity scores such that $s(x_i,y_i) \sim S(x_i,y_i)= \Gamma(\alpha=5 - \frac{1}{3}(\sqrt{x_i^2+y_i^2})^{0.8}, \theta=2.0)$, where $\mathbb{E}[s(x_i,y_i)] = 10 - \frac{2}{3}(\sqrt{x_i^2+y_i^2})^{0.8}$ and $\textrm{Var}[s(x_i,y_i)]=20 - \frac{4}{3}(\sqrt{x_i^2+y_i^2})^{0.8}$. To estimate $\epsilon$, we randomly sampled several pairs of points in the unit circle, then computed $\frac{\tvd (S(x_1,y_1),S(x_2,y_2))}{||(x_1,y_1)-(x_2,y_2)||}$ and took the maximum value computed to be $\epsilon$. Then, we sampled several $N_\textrm{calib}$ values and $\rho$ values on a log scale and computed the coverage across 100 trials. In each trial, we randomly sampled $N_\textrm{calib}$ points on the unit circle, then sampled a non-conformity score from the distribution defined by the coordinates of each sampled point. Since we know the distribution, we can directly compute the TV distance \eqref{eq:init_barber} from \citet{barber2023conformal}. We also computed the tight bound in \eqref{eq:thm3_bound} using the estimated $\epsilon$ value. 

\begin{figure}[htbp]
    \centering
    \includegraphics[width=\textwidth]{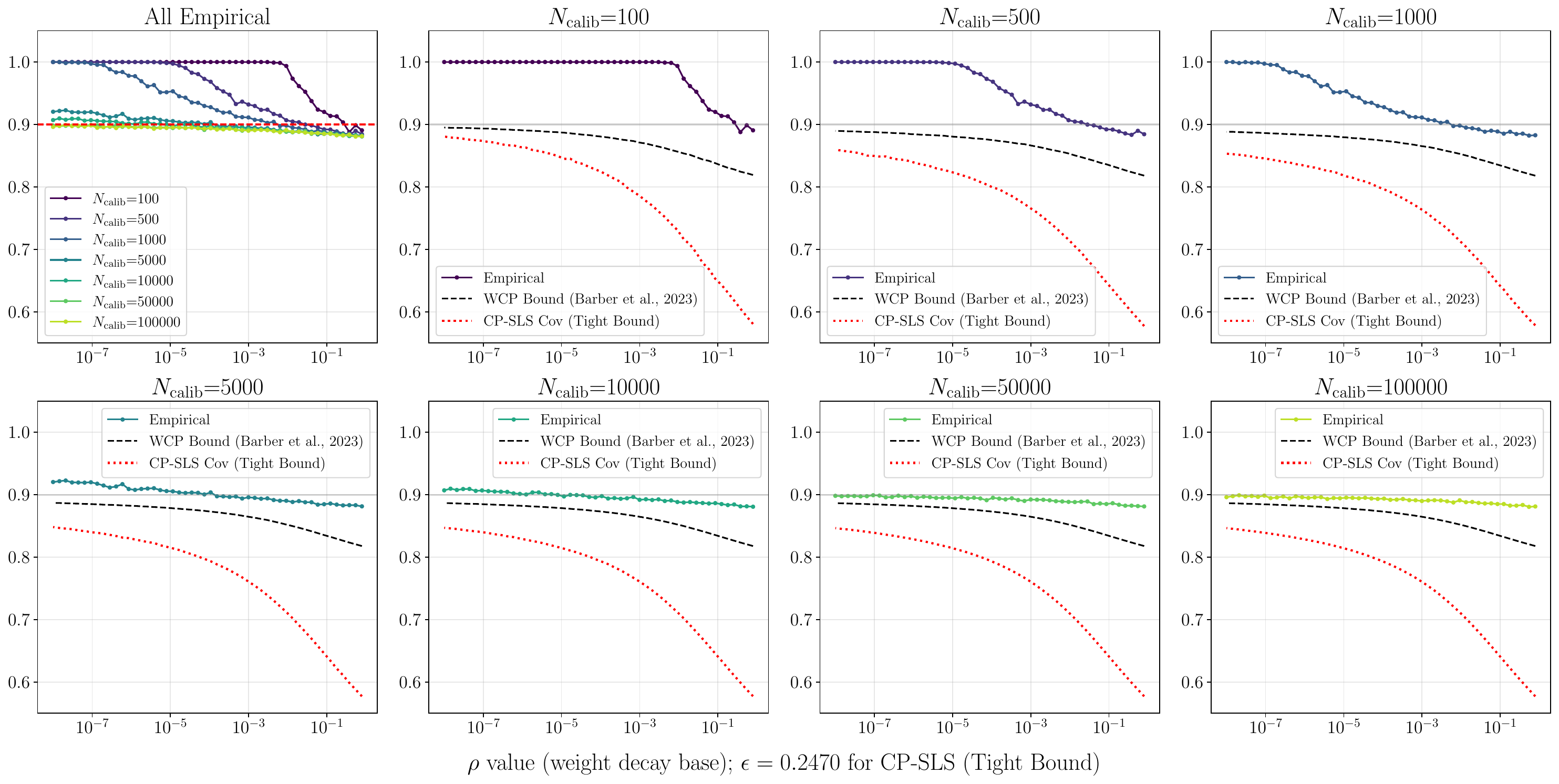}
    \caption{\textbf{Toy Example.} We plot the empirical coverage, along with the theoretical bounds from \eqref{eq:init_barber}, \citet{barber2023conformal}, and our derived bound (labeled) as CP-SLS cov in \eqref{eq:thm3_bound} for $\rho$ values selected on a log scale and several $N_\textrm{calib}$ values. We observe that the lower bounds and empirical coverage increase as $\rho$ decreases, and that the bounds increase more rapidly for smaller $\rho$ values.}
    \label{fig:toy_example_empirical}
\end{figure}

The results in Figure \ref{fig:toy_example_empirical} demonstrate that the lower bound on coverage ($=1-\alpha - \textit{Coverage Gap}$) increases for large $\rho$ as well as the empirically observed coverage. Additionally, the tightness of our derived bound, relative to \citet{barber2023conformal}, improves for smaller $\rho$ values. Similarly, we find the lower-bound increases for fixed $\rho$ values as $N_\textrm{calib}$ increases. While these results suggest a smaller $\rho$ is optimal, for coverage, it is essential to consider the utility sacrifice. Smaller $\rho$ values lead to larger, potentially infinite, values being computed as the $1-\alpha$ quantile, thereby making the uncertainty set harder to use and potentially impractical.  

\newpage
\section{Additional Results}\label{app:additional_results}

\begin{figure}[!h] 
    \centering
    
    \begin{minipage}{1.0\textwidth}
        \centering
    \includegraphics[width=\textwidth]{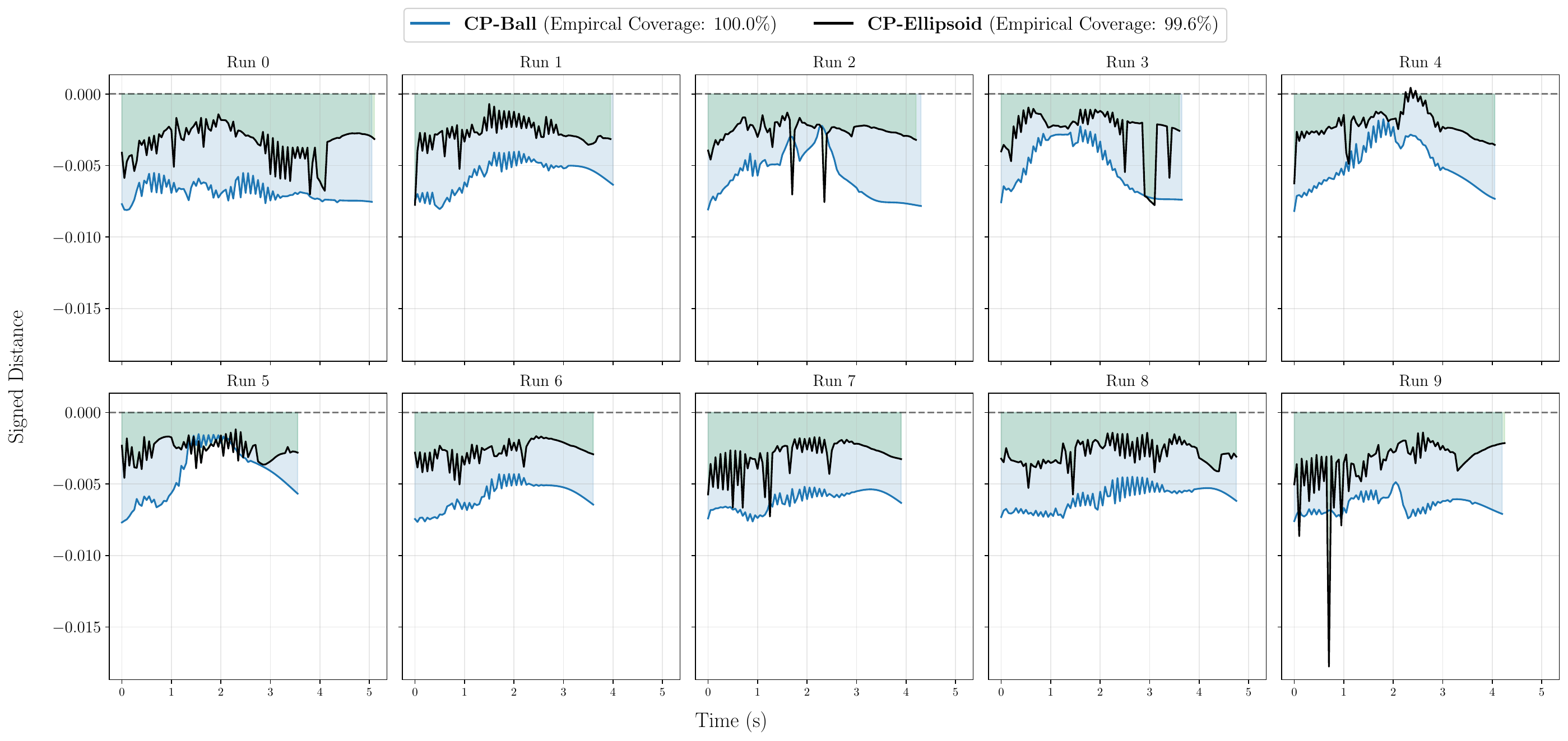}
    \caption{\textbf{In Domain Car.} We plot the minimum distance of the prediction error to the edge of the ellipsoid. Both the fixed ball and adaptive ellipsoid method empirically satisfy the coverage guarantee and remain inside over $99.33\% = (1- \frac{0.1}{15})\times100\%$ of the time.}
    \label{fig:ellip_in_domain}
    \end{minipage}

    
    \begin{minipage}{1.0\textwidth}
        \centering
    \includegraphics[trim={0cm 0.5cm 0cm 1.3cm}, clip, width=\textwidth]{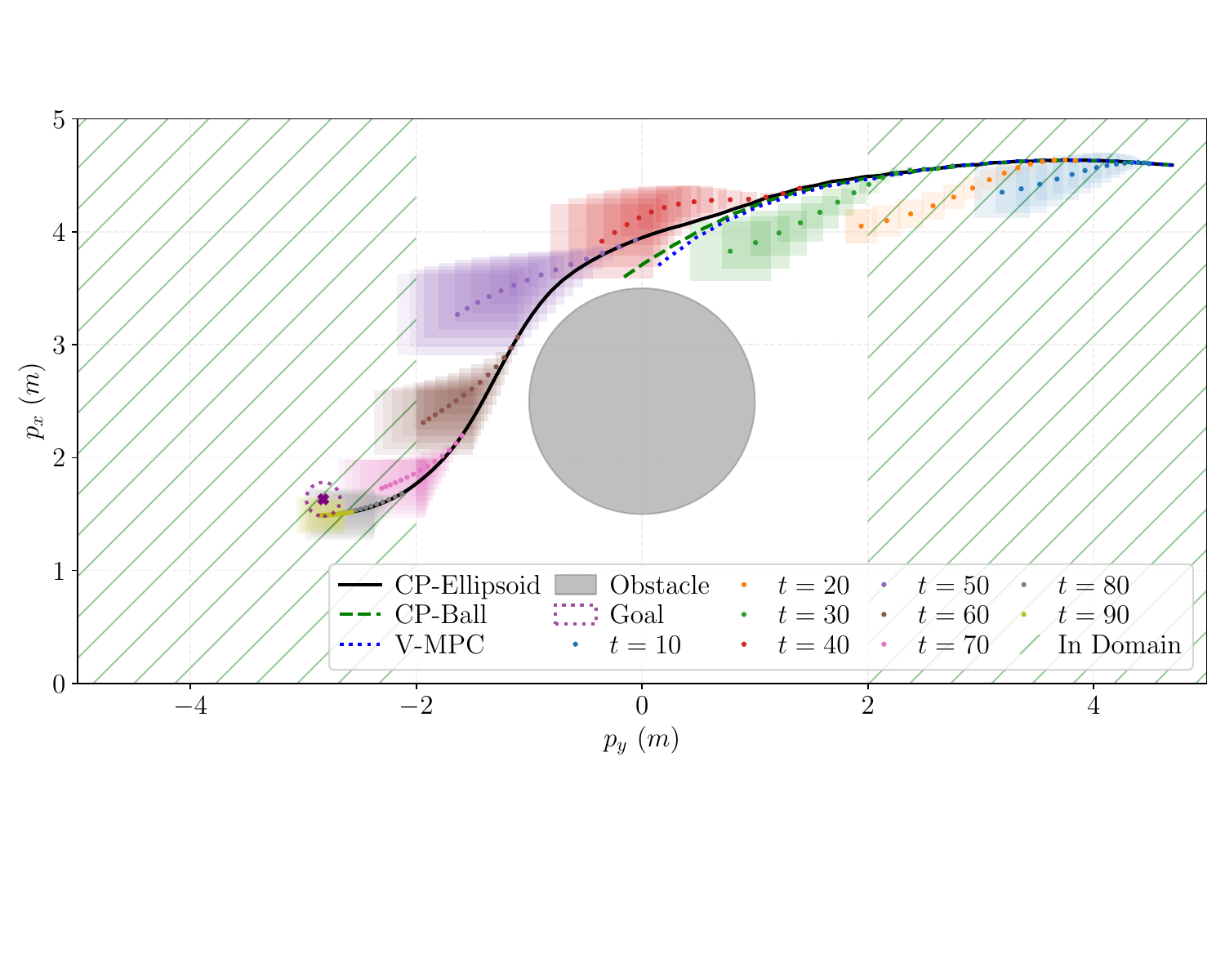}
    \vspace{-3cm}
    \caption{\textbf{Friction Car.} We provide an example run of the friction area where the adaptive ellipsoid approach is successful in avoiding the obstacle. The remaining approaches fail because they do not maintain sufficient obstacle proximity and thus crash into the obstacle.}
    \label{fig:second_run_figure}
    \end{minipage}
\end{figure}
\clearpage

\begin{table}[h]
    \small
  \centering
  \caption{\textbf{Active Uncertainty.} The computation times (ms), and prediction errors when planning with vanilla-MPC, the adaptive ellipsoid with and without active uncertainty. }
  \label{tab:active_uncert}
  \begin{tabular}{l|ccc}
    \hline
    \textbf{Model} & \textbf{V-MPC} & \textbf{CP-Ellipsoid} & \textbf{CP-Ellipsoid + Active Uncertainty} \\
    \hline
    \textbf{Time (ms)} & 56.1 $\pm$ 1.1 & 132.0 $\pm$ 9.9 & 131.4 $\pm$ 8.2 \\
    \hline
    \textbf{Pred. Err. ($L_2$) ($\times 10^{-2}$)} & 2.82 $\pm$ 7.85 & 2.83 $\pm$ 7.89 & .34 $\pm$ .59 \\
    \hline
  \end{tabular}
\end{table}


\begin{figure}[p] 
    \centering
    
    \begin{minipage}{1.0\textwidth}
        \centering
    \includegraphics[trim={0cm 2cm 0cm 2cm}, clip,width=0.95\textwidth]{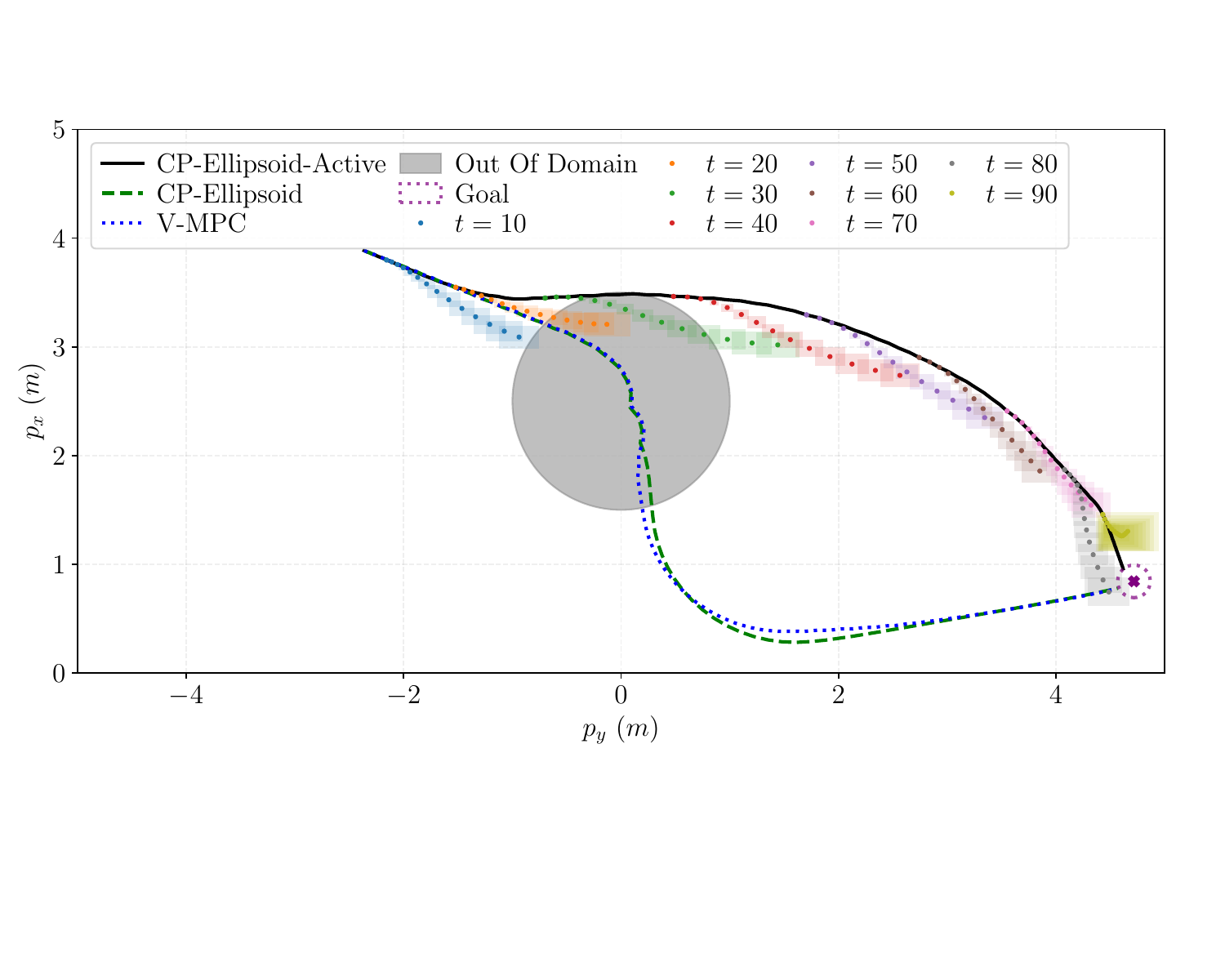}
    \vspace{-1.8cm}
    \caption{\textbf{Active Uncertainty.} We provide an example run of the rollouts when using V-MPC and CP-Ellipsoid with active uncertainty. We find that the active uncertainty approach mostly avoids the OOD region, and thus has a smoother, more certain trajectory.}
    \label{fig:rollout_example_1_active_uncert}
    \end{minipage}

    \begin{minipage}{1.0\textwidth}
        \centering
    \includegraphics[trim={0cm 2cm 0cm 2cm}, clip,width=0.95\textwidth]{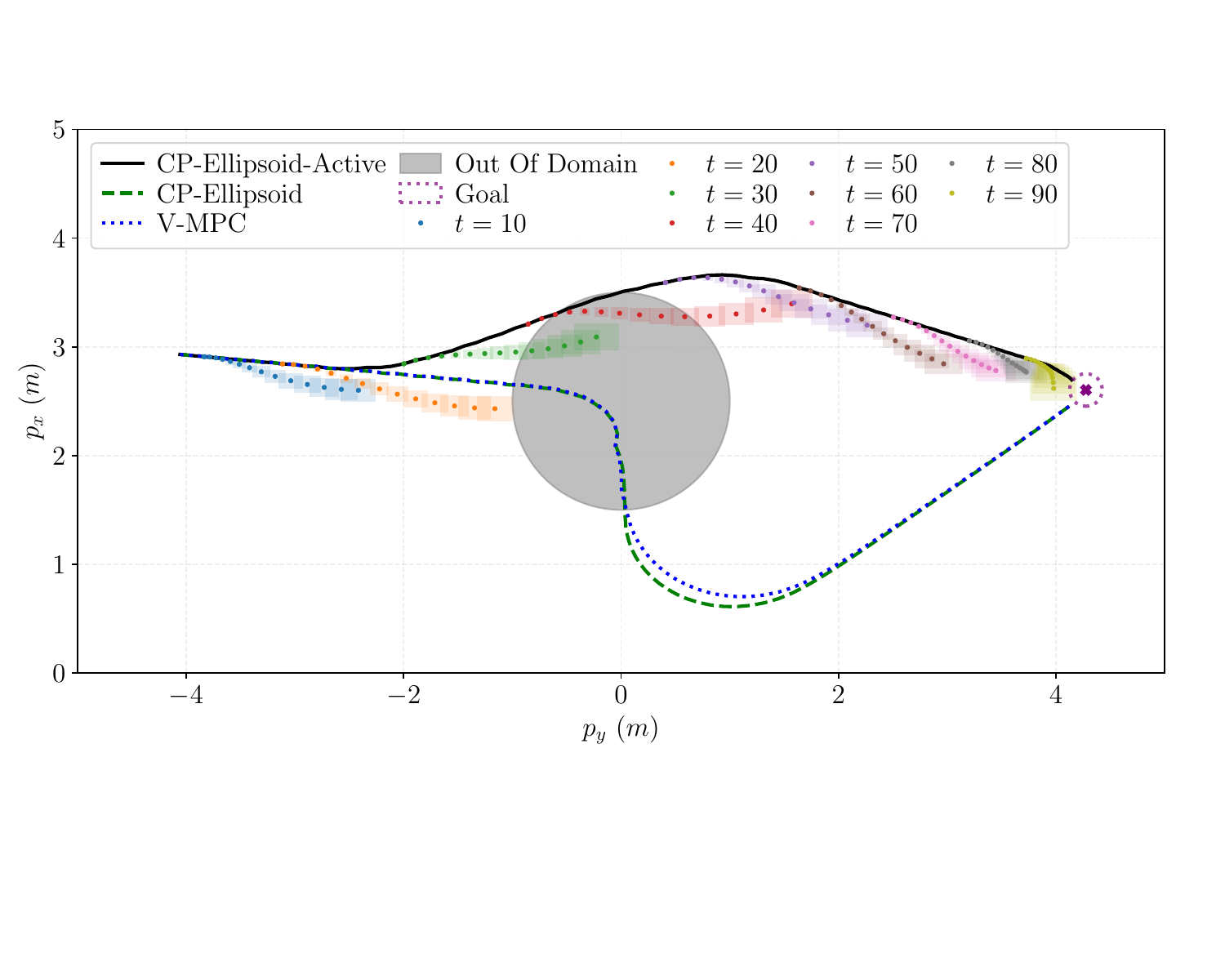}
    \vspace{-1.8cm}
    \caption{\textbf{Active Uncertainty.} We provide another example run of the rollouts when using V-MPC and CP-Ellipsoid with active uncertainty. Similar to Figure \ref{fig:rollout_example_1_active_uncert}, the active uncertainty rollout is smoother by mitigating time in the OOD region. Although the active uncertainty approach has more steps in the OOD region than shown in Figure \ref{fig:rollout_example_1_active_uncert}, it is near the in-distribution region and thus encounters limited OOD errors during the rollout.}
    \label{fig:rollout_example_2_active_uncert}
    \end{minipage}
\end{figure}
\begin{figure}[!h]
    \centering
    \includegraphics[width=\textwidth]{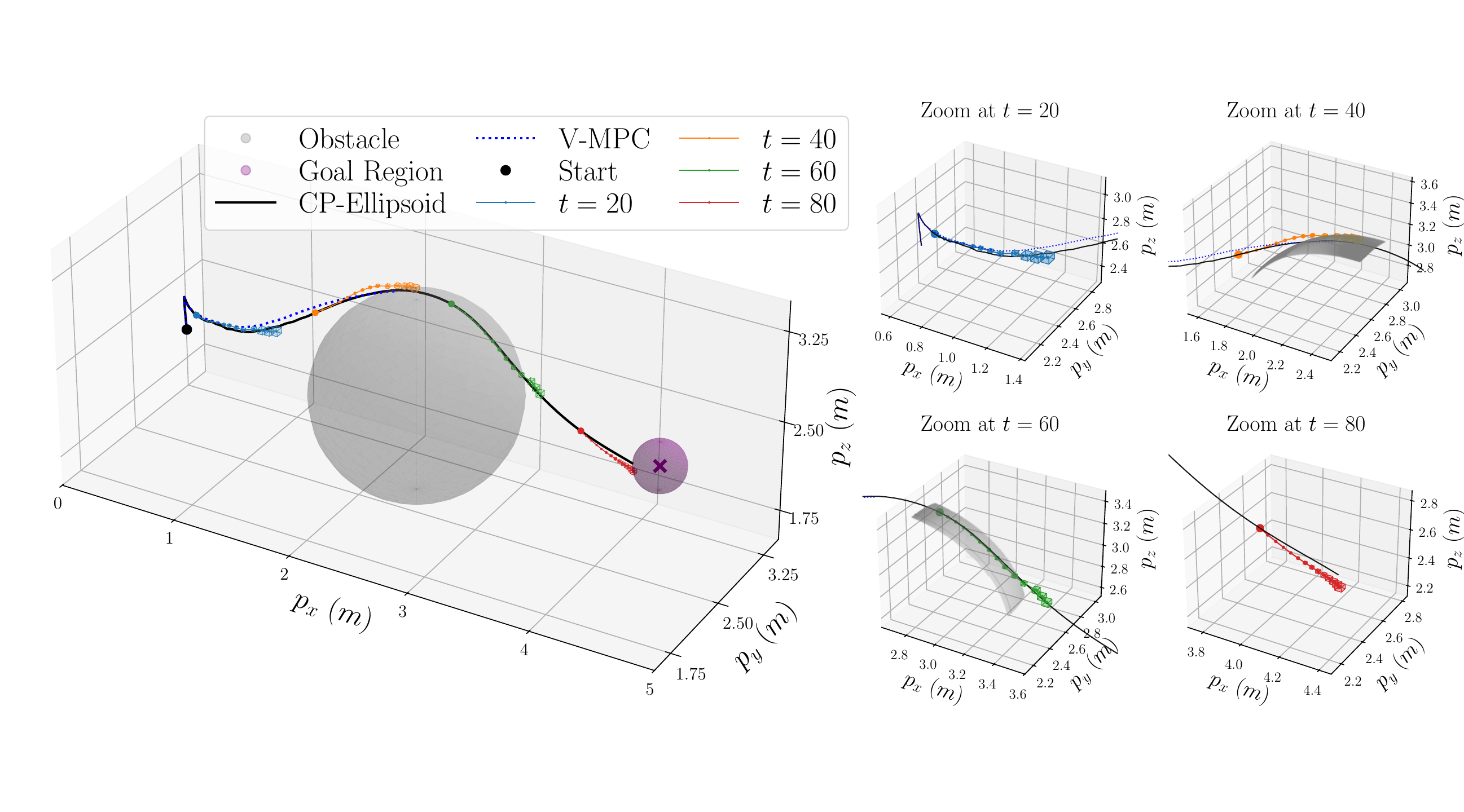}
    \caption{\textbf{Quadcopter.} We provide another example run of the quadcopter rollout, with a different initial condition than in Figure \ref{fig:quad_trajectory}. Again, we observe the quadcopter is successful in avoiding the obstacle when using CP-Ellipsoid (unlike V-MPC) when entering the OOD region. }
    \label{fig:quad_traj_additional_result}
\end{figure}
\end{document}